\documentclass[sigconf]{acmart}
\AtBeginDocument{%
  \providecommand\BibTeX{{%
    \normalfont B\kern-0.5em{\scshape i\kern-0.25em b}\kern-0.8em\TeX}}}

\acmPrice{15.00}

\copyrightyear{2023}
\acmYear{2023}
\setcopyright{rightsretained}
\acmConference[KDD '23]{Proceedings of the 29th ACM SIGKDD Conference on Knowledge Discovery and Data Mining}{August 6--10, 2023}{Long Beach, CA, USA}
\acmBooktitle{Proceedings of the 29th ACM SIGKDD Conference on Knowledge Discovery and Data Mining (KDD '23), August 6--10, 2023, Long Beach, CA, USA}
\acmDOI{10.1145/3580305.3599417}
\acmISBN{979-8-4007-0103-0/23/08}
\usepackage{etoolbox}
\makeatletter
\patchcmd{\maketitle}{\@copyrightpermission}{
   \begin{minipage}{0.3\columnwidth}
     \href{https://creativecommons.org/licenses/by/4.0/}{\includegraphics[width=0.90\textwidth]{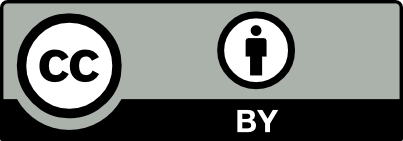}}
   \end{minipage}\hfill
   \begin{minipage}{0.7\columnwidth}
     \href{https://creativecommons.org/licenses/by/4.0/}{This work is licensed under a Creative Commons Attribution International 4.0 License.}
   \end{minipage}
 
   \vspace{5pt}
}{}{}

\makeatother
\usepackage{colortbl}
 \usepackage{enumitem}
\usepackage{svg}
\usepackage{amsmath}
\usepackage{caption}
\usepackage{subcaption}
\usepackage{multirow}
\usepackage{algorithm}
\usepackage{algcompatible}
\usepackage{graphicx} 
\usepackage{algpseudocode}
\usepackage{hyperref}
\usepackage{stfloats}

\DeclareMathOperator*{\argmax}{arg\,max}
\DeclareMathOperator*{\argmin}{arg\,min}
\newcommand{\probP}{\text{I\kern-0.15em P}}
\algrenewcomment[1]{\(\triangleright\) #1}
\newtheorem{prop}{Proposition}
\theoremstyle{plain}
\newtheorem{theorem}{Theorem}[section]

\theoremstyle{definition}
\newtheorem{definition}[theorem]{Definition}

\theoremstyle{remark}

\algnewcommand\RETURN{\algorithmicreturn}

\begin{document}
\newcommand{\rongzhi}[1]{\textcolor{blue}{[Rongzhi: #1]}}
\newcommand{\xc}[1]{\textcolor{cyan}{[Jiaming: #1]}}
\newcommand{\ours}{\textsc{LocalBoost}}
\synctex=1
\settopmatter{printacmref=true}
\newcommand{\zc}[1]{\textcolor{blue}{[#1 - Chao]}}
\newcommand{\yy}[1]{\textcolor{magenta}{[#1 - Yue]}}
\newcommand{\jmshen}[1]{\textcolor{red}{[Jiaming: #1]}}

\title{Local Boosting for Weakly-Supervised Learning}

\author{Rongzhi Zhang}
\affiliation{%
  \institution{Georgia Institute of Technology}
  \city{Atlanta}
  \state{GA}
  \country{USA}}
\email{rongzhi.zhang@gatech.edu}

\author{Yue Yu}
\affiliation{%
  \institution{Georgia Institute of Technology}
  \city{Atlanta}
  \state{GA}
  \country{USA}}
\email{yueyu@gatech.edu}

\author{Jiaming Shen}
\affiliation{%
  \institution{Google Research}
  \city{New York City}
  \state{NY}
  \country{USA}}
\email{jmshen@google.com}

\author{Xiquan Cui}
\affiliation{%
  \institution{The Home Depot}
  \city{Atlanta}
  \state{GA}
  \country{USA}}
\email{xiquan_cui@homedepot.com}

\author{Chao Zhang}
\affiliation{%
  \institution{Georgia Institute of Technology}
  \city{Atlanta}
  \state{GA}
  \country{USA}}
\email{chaozhang@gatech.edu}
\newcommand{\yue}[1]{{\text{\textcolor{red}{[Yue: #1]}}}}
\renewcommand{\shortauthors}{Rongzhi Zhang et al.}

\begin{abstract}
Boosting is a commonly used technique to enhance the performance of a set of base models by combining them into a strong ensemble model. Though widely adopted, boosting is typically used in supervised learning where the data is labeled accurately. However, in weakly supervised learning, where most of the data is labeled through weak and noisy sources, it remains nontrivial to design effective boosting approaches. 
  In this work, we show that the standard implementation of the convex combination of base learners can hardly work due to the presence of noisy labels. 
  Instead, we propose \ours, a novel framework for weakly-supervised boosting. 
\ours~iteratively boosts the ensemble model from two dimensions, \emph{i.e., intra-source and inter-source}. 
The intra-source boosting introduces locality to the base learners and enables each base learner to focus on a particular feature regime by training new base learners on granularity-varying error regions. 
For the inter-source boosting, we leverage a conditional function to indicate the weak source where the sample is more likely to appear. 
To account for the weak labels, we further design an estimate-then-modify approach to compute the model weights. 
Experiments on seven datasets show that our method significantly outperforms vanilla boosting methods and other weakly-supervised methods.

\end{abstract}

\ccsdesc[500]{Computing methodologies~Boosting}


\keywords{Boosting Methods, Weak Supervision, Classification}



\maketitle

\section{Introduction}

\emph{Weakly-supervised learning} (WSL) has gained significant attention as a solution to the challenge of label scarcity in machine learning. WSL leverages weak supervision signals, such as labeling functions or other models, to generate a large amount of weakly labeled data, which is easier to obtain than complete annotations. Despite achieving promising results in various tasks including text classification \cite{awasthi2020learning}, sequence tagging \cite{safranchik2020weakly}, and e-commerce~\cite{zhang2022adaptive}, an empirical study ~\citep{zhang2021wrench} reveals that even state-of-the-art WSL methods still underperform fully-supervised methods by significant margins, where the average performance discrepancy is 18.84\%, measured by accuracy or F1 score.

\begin{figure}
  \vspace{-7.5pt}
  \centering
  \includegraphics[scale=0.5]{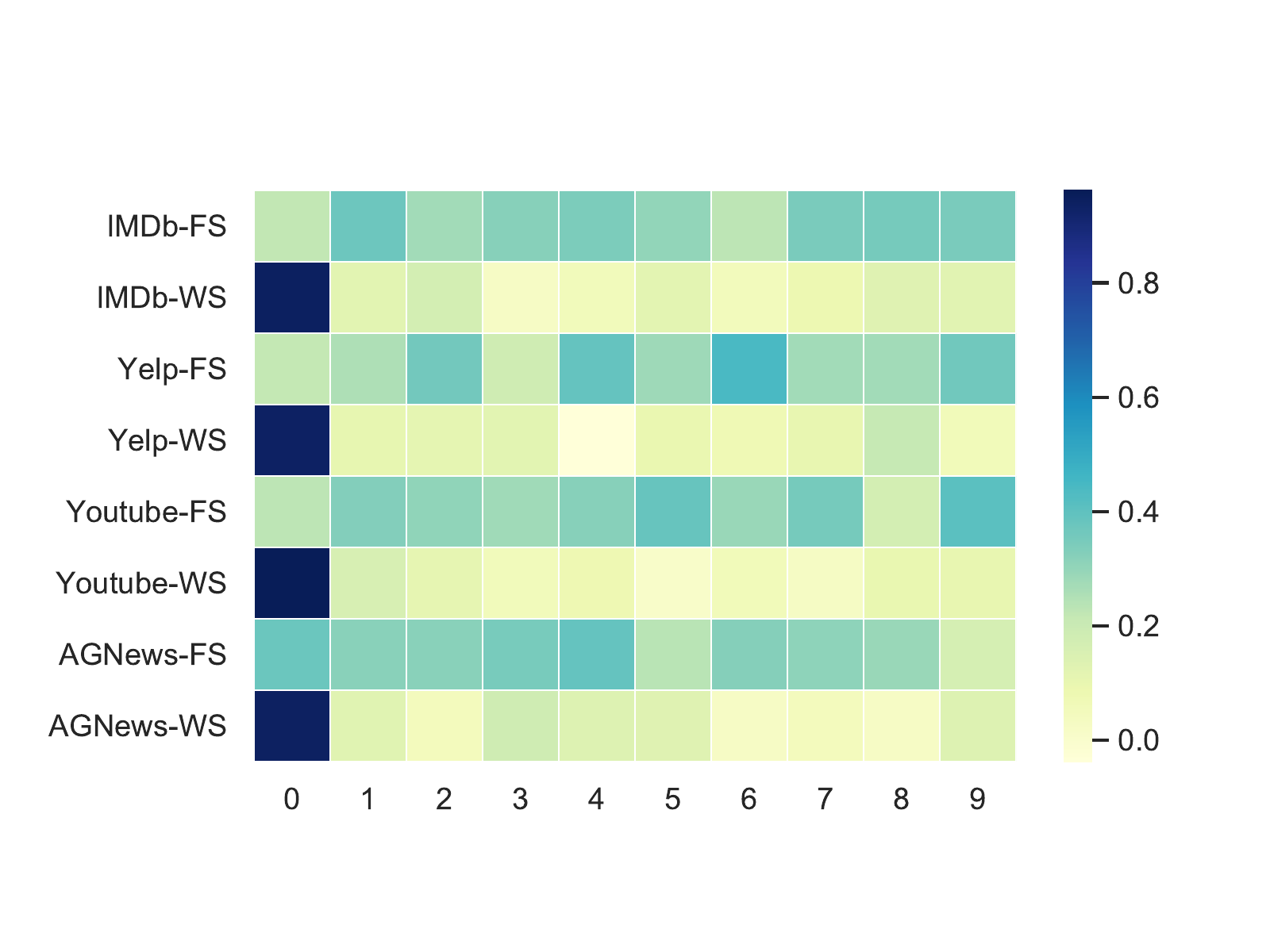}
  \vspace{-7.5ex}
  \caption{Heatmap of base learner weights in model ensembles. Suffix ``-FS'' indicates fully-supervised settings using clean labels and suffix ``-WS'' indicates weakly-supervised settings. Each model ensemble consists of 10 base learners, and their weights are shown by color in the heatmap.}
  \label{fig:dominate}
\end{figure}



On the other hand, boosting algorithm is one of the most commonly used approaches to enhance the performance of machine learning models by combining multiple base models~\cite{freund1995desicion, freund1996experiments, mason1999boosting, chen2016xgboost, ratsch2005efficient}.
For example, AdaBoost~\cite{freund1995desicion} dynamically adjusts the importance weight of each training example to learn multiple base models and uses a weighted combination to aggregate these base models' predictions. 
XGBoost~\cite{chen2016xgboost} iteratively computes the gradients and hessians defined on a clean training set to fit base learners and combines their predictions via weighted summation. 
Despite the encouring performance, these boosting algorithms usually assume the availability of a clean labeled dataset.  
In WSL, however, 
the imperfect supervision signals interfere with the training data importance reweighting which further prevents us from computing an accurate weight of each base learner.
If we naively apply these supervised boosting methods using a weakly labeled dataset, we observe a phenomenon called ``\emph{weight domination}'' where the assigned weight of the initial base model is too large and dominates the ensemble model prediction, as shown in Figure~\ref{fig:dominate}.


A key challenge of adapting boosting methods to the WSL setting is to accurately compute the importance of each example in the weakly labeled training data for each base learner.
Previously, when a clean dataset is provided, the goal of data importance reweighing process is to prioritize instances with large errors for subsequent base learner training.
This effectively localizes the base learner to the error region in the label space. 
However, in WSL, the noisy labels hinder the accurate identification of error instances and thus we need to shift our focus from the label space to the training data space. 
One potential approach is to partition the weakly-labeled training data into subsets and constructs a mixture of expert models (MoE)~\cite{jacobs1991adaptive} where each expert is localized for one training data subset.
Along this line, \citet{tsai2021mice} propose partitioning the unlabeled dataset into latent semantics subsets and using multiple expert models to discriminate instances.
However, this approach assumes the input data naturally reside in a homogeneous feature space and requires a hyper-parameter search to appropriately localize the expert models. Additionally, the off-the-shelf clusters do not adapt during the learning process, which conflicts with the philosophy of boosting methods.



We investigate the problem of boosting in the context of weakly supervised learning, where most of the data are labeled by weak sources and only a limited number of data points have accurate labels. To address the difficulties posed by this setting, we introduce \ours, a novel iterative and adaptive framework for WSL boosting. \ours\ retains the essential concepts of the traditional boosting approach while incorporating adaptations specifically designed for the WSL scenario, described as follows:

\begin{itemize}[leftmargin=0.5cm]
  \item \emph{Base Learner Locality}.
Motivated by the challenges posed by the data reweighting approach in AdaBoost for WSL and the limitations of hard clustering in MoE methods, we propose a new approach to base learner localization. In AdaBoost, large-error instances are assigned with larger weights for model training, however, this approach does not account for the fact that error instances exist in  multiple feature regimes that are difficult to capture with weak labels. 
Additionally, the rigid clusters and fixed expert models in MoE cannot adapt in the iterative learning process, hindering the framework's ability to dynamically target weak feature regimes and build upon preceding models. 
To address these issues, our proposed framework \ours\ assigns base learners to adaptively updated local regions in the embedding space, thereby introducing locality to the base learners.


  \item \emph{Two-dimension Boosting}.
  Effective aggregation of localized base learners in WSL goes beyond the simple convex combination in supervised settings (as shown in Sec.~\ref{sec:cond_func}). 
  To account for potential label noises from weak sources, we aim to learn multiple complementary base learners in \ours.
  To fulfill this goal, we introduce a weighting function to compute the conditional probability of weak sources that are more likely to annotate a given data instance.
  We further design a two-dimensional boosting framework in \ours, where inter-source boosting and intra-source boosting are performed alternately. 
  The former improves the base learners within a given weak source, while the latter complements the base learners with additional models trained from other weak sources.

  \item \emph{Interactions between Weak and Clean Labels}.
  We incorporate the interactions between weak and clean labels into \ours~framework in two steps: 
  (1) We compute a mapping between the small clean dataset and the large weakly labeled dataset to localize base learners in the data embedding space. We first identify the errors made by the current model ensemble, and then sample corresponding clusters from the large weak dataset to form the training set for the next base learner. 
  (2) We propose a novel estimate-then-modify approach for computing base learner weights. Initially, the weights are estimated on the large weakly labeled dataset.
  Then, we refine these estimates by generating multiple perturbations of the model weights and selecting the one that results in the lowest error rate on the small clean dataset as the modified weights.
  
\end{itemize}

We evaluate \ours~on seven datasets including sentiment analysis, topic classification and relation classification from WRENCH \cite{zhang2021wrench}, the standard benchmark for weakly supervised learning. The results indicate that \ours~achieves superior performance compared with other state-of-the-art methods. 
Moreover, our analysis further confirms the effectiveness of boosting in two dimensions and incorporating interactions between weak and clean labels.
We summarize our key contributions as follows:
\begin{enumerate}[leftmargin=0.5cm]

\item
We present \ours\footnote{The implementation is available at \href{https://github.com/rz-zhang/local\_boost}{https://github.com/rz-zhang/local\_boost}}, a novel weakly-supervised boosting framework that implements progressive inter-source and intra-source boosting.

\item
We incorporate explicit locality into the base learners of the boosting framework, allowing them to specialize in finer-grained data regions and perform well in specific feature regimes.

\item We leverage the interactions between weak and clean labels for effective base learner localization and weight estimation.

  \item We conduct extensive experiments on seven benchmark datasets and demonstrate the superiority of \ours\ over WSL and ensemble baselines.

\end{enumerate}





\section{Related Work}
\noindent\textbf{Weakly Supervised Learning.}~
Weakly supervised learning (WSL) focuses on training machine learning models with a variety of weaker, often programmatic supervision sources~\cite{zhang2021wrench}.
Specifically, in WSL, users provide \emph{weak supervision sources}, e.g., heuristics rules~\cite{zhuang2022resel,xu2023weakly}, knowledge bases~\cite{hoffmann2011knowledge}, and pre-trained models~\cite{smith2022language,yu2023zero,zhang2022prompt}, in the form of \emph{labeling functions (LFs)}, which provide labels for some subset of the data to create the training set.

The main challenge in WSL is incorrect or conflicting labels produced by labeling functions~\cite{zhang2022survey}. To address this, recent research has explored two solutions. The first is to create label models, which aggregate the noisy votes of labeling functions using different dependency assumptions. This approach has been used in various studies, including \cite{ratner2017snorkel,ratner2019training,fu2020fast,zhang2022leveraging,stephan-etal-2022-sepll}. The second solution involves designing end models that use noise-robust learning techniques to prevent overfitting to label noise, as seen in studies by \cite{yu2021fine,wu2020biased,li2021improved}. This current work belongs to the first category.

For label models, several recent works attempt to better aggregate the noisy labeling functions via better modeling the distribution between LF and ground truth labels~\cite{ratner2019training,fu2020fast}, \citet{zhang2022leveraging} incorporate the instance features with probabilistic models, while \cite{ruhling2021end,ren2020denoising,sam2022losses} learn the label aggregation using supervision from the target task.
Besides, \cite{safranchik2020weakly,shin2022universalizing, wu2022towards} adapt the WSL to a broader range of applications such as regression and structural prediction, and \cite{zhang2020seqmix,biegel2021active,yu2022actune,yu2022cold} design active learning approaches to strategically solicit human feedbacks to refine weak supervision models.

So far, only a few works have attempted to integrate the boosting or ensemble methods with WSL.
\citet{guan2018said} learn an individual weight for each labeler and aggregate the prediction results for prediction.
\citet{zhang2022prompt} leverage boosting-style ensemble
strategy to identify difficult instances and
adaptively solicit new rules from human annotators and \citet{zhang2022adaptive} leverage rules from multiple views to further enhance the performance.
Very recently, \citet{zhao2022admoe} use mixture-of-experts (MoE) to route noisy labeling functions to different experts for achieving specialized and scalable learning. However, these works directly {combine} existing ensembling techniques with noisy labels. Instead, we first identify the key drawback of adopting boosting for WSL (Fig.~\ref{fig:dominate}) and then design two-dimension boosting to resolve this issue, which provides a more effective and flexible way for learning with multi-source weak labels.\\

\noindent\textbf{Boosting}.
The boosting methods are prevailing to improve machine learning models via the combination of base learners. 
Research in this area begins in the last century~\cite{freund1995desicion} and numerous variants have been proposed \cite{breiman1999prediction, mason1999boosting, friedman2001greedy, chen2016xgboost}. AdaBoost~\cite{freund1995desicion} incorporates base learners to the model ensemble via iterative data reweighting and model weights computation that minimizes the total error on the training set, to improve the model ensemble for binary classification. XGBoost~\cite{chen2016xgboost} presents an advanced implementation of boosting by using a more regularized model formalization to control over-fitting, and it has the virtue of being accurate and fast. However, these methods are discussed in the context of fully-supervised learning, and their principal derivation or implementation relies on clean data. This results in issues with its adaption to weakly-supervised settings as it is challenging to obtain reliable computation from weakly labeled data.

Recent works explore multi-class boosting ~\cite{brukhim2021multiclass, cortes2021boosting} and the application side of boosting \cite{zhang2022building, huang2018learning, nitanda2018functional, suggala2020generalized}. \citet{brukhim2021multiclass} study the resources required for boosting and present how the learning cost of boosting depends on the number of classes.
\citet{zhang2022building} explore boosting in the context of adversarial robustness and propose a robust ensemble approach via margin boosting.
For the application with deep neural networks, \citet{taherkhani2020adaboost} integrate AdaBoost with a convolutional neural network to deal with the data imbalance.
Among these works, the most related one to ours is MultiBoost~\cite{cortes2021boosting}, where the authors study boosting in the presence of multiple source domains. They put forward a Q-ensemble to compute the conditional probability of different domains given the input data. This formulation can be traced back to multi-source learning in fully supervised settings \cite{mansour2008domain, hoffman2018algorithms, cortes2021discriminative}. In this work, we draw inspiration from such formulation and present an adaptation to the WSL setting -- we design a weighting function to compute the conditional probability of weak sources. In this way, we modulate the base learners to highly relevant weak sources and present a two-dimension boosting accordingly.


\section{Preliminaries}
Let $\mathcal{X}$ denote the input space and $\mathcal{Y}=\{1, \cdots, C\}$ represent the output space, where $C$ is the number of classes. We have a large weakly labeled dataset $\mathcal{D}_l$ and a small clean dataset $\mathcal{D}_c$. The weak labels of $\mathcal{D}_l$ are generated by a set of weak sources $\mathcal{R}$, and the number of weak sources $|\mathcal{R}| = p$.
\begin{definition}[Weak sources]
    In the context of WSL, the weak sources refer to labeling functions (LFs), which are constructed via keywords or semantic rules.
    In this work, we will use the term ``weak source'' and ``labeling function'' interchangeably. 
    Given an unlabeled data sample $\boldsymbol{x}_u$, a weak source $r(\cdot)$ maps it into the label space: $r(\boldsymbol{x}_u) \rightarrow y \in \mathcal{Y} \cup\{0\}$. Here
    $\mathcal{Y}$ is the original label set for the task and $\{0\}$ is a special label indicating $\boldsymbol{x}_u$ is unmatchable by $r(\cdot)$. Given a set of $N$ samples and $p$ labeling functions, we can obtain a LF matching matrix $A_{N \times p}$, where each entry $a_{i,j}\in \{0, 1\}$ denotes whether the $i$-th sample is matched by the $j$-th LF.
\end{definition}

\begin{definition}[Base learner]
    Let $S=\left\{\left(\mathbf{x}_i, y_i\right)\right\}_{i=1}^n$ be $n$ i.i.d samples drawn from $\probP$, and $\probP_n$ be the empirical distribution of $\mathcal{S}$. We let $h:\mathcal{X} \rightarrow \mathcal{Y}$ denote a base learner that predicts the label of $\mathbf{x}$ as $h(\mathbf{x})$. For a base learner $h$ and a distribution $\mathcal{S}$, we denote the expected loss of $h$ as $\mathcal{L}(\mathcal{S}, h)=\mathbb{E}_{(x, y) \sim \mathcal{S}}[\ell(h(x), y)]$, where $\ell$ is cross entropy loss.
\end{definition}

\begin{definition}[Ensemble]
    Let $\mathcal{H}$ be the set of base learners, and   $h:\mathcal{X} \rightarrow \mathcal{Y}, \forall h \in \mathcal{H}$. In ensembling, we place a probability distribution $\probP_q$ over $\mathcal{H}$ to assign weights of base learners. It leads to a score-based base learner $h_{\probP_q}: \mathcal{X} \rightarrow \mathbb{R}^C$ with $\left[h_{\probP_q}(\mathbf{x})\right]_j=\mathbb{E}_{h \sim {\probP_q}}[\mathbb{I}(h(\mathbf{x})=j)]$, and we convert it into a base learner in the form of a standard classifier with the argmax function: $f(\mathbf{x})=\operatorname{argmax}_{j \in \mathcal{Y}}\left[h_{\probP_q}(\mathbf{x})\right]_j$. Given the score-based $\mathcal{H}_f=\left\{f_1, f_2, \ldots\right\}$, where $f: \mathcal{X} \rightarrow \mathbb{R}^C, \forall f \in \mathcal{H}_f$, we can have a convex combination of the base learners via a set of real-valued weights $W$ over $\mathcal{H}_f$, and define the weighted ensemble $F$ as $F_j(\mathbf{x})=\sum_{f \in \mathcal{H}_f} W(f)f_j(\mathbf{x})$, where $W(f) \in \mathbb{R}$ is the weight of the base learner $f$.
\end{definition}

\paragraph{Problem Formulation.}
Given a large weakly labeled dataset $\mathcal{D}_l$, a small clean dataset $\mathcal{D}_c$,
and $p$ weak sources, we aim to iteratively obtain base learners $f_j$ to boost the performance of the ensemble model $F_j(\mathbf{x})=\sum_{f \in \mathcal{H}_f} W(f)f_j(\mathbf{x})$.
\section{Methodology}
\subsection{Learning Procedure Overview}
\ours\ iteratively implements inter-source boosting and intra-source boosting in the WSL setting. For weak source $l: 1,\cdots,p$ and iteration $t:1,\cdots,T$, we have a base learner $f_{t,l}(x)$ with a model weight $\alpha_{t,l}$. As we will illustrate in Sec.~\ref{sec:cond_func}, the straightforward convex combination may not work for the WSL setting. To this end, we introduce a conditional probability function $Q(l|x)$ to account for the fact that the weak labels are generated from different weak sources. For a given sample $\mathbf{x}$, the base learners from the weak source where $\mathbf{x}$ is more likely to be labeled are allocated higher weights in the voted combination. 

Given a large weakly labeled dataset $\mathcal{D}_l$ and a small clean
dataset $\mathcal{D}_c$, the overall learning procedure runs as follows: In iteration $(t,l)$, the preceding ensemble model $F_{t-1, l}(\mathbf{x})$ or $F_{t, l-1}(\mathbf{x})$ accumulated errors on $\mathcal{D}_c$, say the large-error instances are $s_1,\cdots,s_k$. Then we sample $k$ clusters $S_1,\cdots, S_k$ in $\mathcal{D}_l$ near $s_1,\cdots,s_k$, these clusters form a training set $\mathcal{S}_{t,l} = \{S_1, \cdots, S_k\}$ for the base learner $f_{t,l}(\mathbf{x})$ and we fit $f_{t,l}(\mathbf{x})$ on $\mathcal{S}_{t,l}$. Next, we estimate the model weights $\alpha_{t,l}$ on $\mathcal{D}_l$ and modify it on $\mathcal{D}_c$. After $T$ iterations, we obtain the final ensemble:
\begin{equation}
    F(x) = \sum_{t=1}^T \sum_{l=1}^p \alpha_{t, l} \mathrm{Q}(l \mid x) f_{t, l}(x).
\end{equation}
The key components presented in the above model ensemble, including the model weights $\alpha_{t,l}$, the conditional function $Q(l|x)$, and the local base learner $f_{t,l}(x)$, are specially designed for the WSL setting. 
Different from the formulation of Adaboost~\cite{freund1995desicion} and Q-ensemble~\cite{cortes2021boosting, cortes2021discriminative} that are solely based on the clean labels, \ours\ relieves such reliance on fully accurate supervision via the interaction between weak labels and clean labels, and localizes the base learners to adaptively complement the preceding ensemble. 

\begin{algorithm}[!t]
\caption{Pseudo-code of \ours\ Framework}
\label{tab:algo}
\begin{algorithmic}
\Require \text{Large weakly-labeled dataset} $\mathcal{D}_l$, \text{Small clean dataset} $\mathcal{D}_c$,
$p$ \text{weak sources, Source-index dataset} $D_s$
\State \Comment{\textbf{Initialization}}
\State \text{Learn} $Q(l|x)$ \text{on} $\mathcal{D}_s$ \text{({Sec.~\ref{sec:cond_func}})}
\State \text{Fit} $f_{1,1}$ \text{on} $\mathcal{D}_l$,
$\alpha_{1,1} \gets 1$
\State $F_{1,1}(x) \gets \alpha_{1,1}Q(l|x)f_{1,1}(x)$
\For{$t=1$ $\gets$ $T$} \quad\quad  \textit{$\triangleright$ Intra-source boosting}
    \For{$l=1$ $\gets$ $p$} \quad\quad \textit{$\triangleright$ Inter-source boosting}
    \State  $F_{t,l}(x)$ \text{inference on} $\mathcal{D}_c$
    \State \Comment{\textbf{Localize the base learner}} \text{({Sec.~\ref{sec:base-learner}})}
    \State  \text{Identify the large-error instances} $s_1, \cdots, s_k$ \text{on} $\mathcal{D}_c$
    \State  \text{Sample} $k$ \text{clusters} $S_1,\cdots,S_k$ on $\mathcal{D}_l$ \text{based on} $s_1, \cdots, s_k$
    \State Fit $f_{t,l}(x)$ on $\mathcal{D}_{t,l} = \{S_1\cup \cdots\cup S_k\}$
    \State \Comment{\textbf{Weights estimate and correction}} \text{{Sec.~\ref{sec:weight}})}
    \State \text{Estimate} $\alpha_{t, l}$ \text{on} $\mathcal{D}_l$
    \State \text{Generate a group of perturbed} $\mathbf{v}_{t, l} = [\alpha_{1,1},\cdots,\alpha_{t,l}]$ \State \text{Normalize each of them} 
    \State \text{Select} $\argmin_{ \mathbf{v}_{t, l} \in \mathcal{V}_{t, l}}\operatorname{err}_{\mathcal{D}_c}$ as the corrected weight
    \State \Comment{\textbf{Update the ensemble}}
    \State $ F_{t,l}(x) = \sum_{t} \sum_{l} \alpha_{t, l} \mathrm{Q}(l \mid x) f_{t, l}(x)$
    \EndFor
\EndFor\\
\Return \text{Ensemble model} $F(x)$
\end{algorithmic}
\end{algorithm}

The learning algorithm is presented in Algorithm~\ref{tab:algo}. 
In the initialization stage, we prepare a source-index dataset $\mathcal{D}_s$ using the LF matching matrix (Sec.~\ref{sec:cond_func}) to learn the conditional function, and fit the first base learner on the large weakly labeled dataset $\mathcal{D}_l$.
Then we iteratively perform the two-dimension boosting. The inner loop over $p$ weak sources corresponds to the inter-source boosting, while the outside loop over $T$ iterations corresponds to the intra-source boosting. In each iteration $(t, l)$, we start from the ensemble inference on $\mathcal{D}_c$ to identify the accumulated large-error instances $s_1,\cdots,s_k$. These error instances based on the ground-truth labels can accurately reflect the feature regimes where the current ensemble performs poorly. 

Fig.~\ref{fig:local} presents an illustrative visualization of the base learner localization. To localize the subsequent base learners, we sample $k$ clusters from the large weakly labeled dataset $\mathcal{D}_l$ based on the identified large-error instances $s_1,\cdots,s_k$. 
In other words, these instances from $\mathcal{D}_c$ guide the localization of the base learner via a mapping between the instances on the small clean dataset and the regions on the large weakly labeled dataset. To complement the previous ensemble, we fit the new base learner $f_{t,l}(x)$ on the dataset $\mathcal{D}_{t,l}$ consisting of instances from clusters $S_1,\cdots S_k$. 

For the model weights computation, we propose an estimate-then-modify paradigm (as shown in Fig.~\ref{fig:weight_esco}) to leverage both the weak labels and clean labels. On the large weak dataset $\mathcal{D}_l$, we retain the AdaBoost principles to compute the weighted error and yield an estimate of $\alpha_{t,l}$. Considering the labels are obtained from weak sources,  the estimated $\alpha_{t,l}$ on $\mathcal{D}_l$ can hardly guarantee the boosting progress, so we further modify it on the small clean validation dataset $\mathcal{D}_c$. Specifically, we generate a group of the \emph{perturbed} weight vectors $\mathcal{V}_p = \{\mathbf{v}_{t, l}\}^{n_p}$ and compute the weighted error on $\mathcal{D}_c$. Among the perturbations, we select the one that achieves the lowest error as the modified model weights. 

In the following sections, we will first introduce the conditional function designed for the weak sources (Section~\ref{sec:cond_func}), then illustrate how we introduce locality to base learners (Section~\ref{sec:base-learner}), and finally discuss the estimate-then-modify paradigm for the computation of the model weights in WSL settings (Section~\ref{sec:weight}).

\subsection{Conditional Function for Weak Sources Localization}\label{sec:cond_func}
In the WSL setting, the weak labels are generated by weak sources such as labeling functions. We first pinpoint that the standard convex combination of the base learners can lead to a poor ensemble without taking the weak sources into account. 
Being aware of this, we propose to use a conditional function to account for the weak sources. Given an input instance, the conditional function represents the probability of the instance being labeled by each weak source. In this way, the ensemble model is modulated by this conditional function, while being weighted by a series of base learner weights, to deal with the weak sources in the WSL setting. 
\begin{prop}\label{prop:paradox}
    There exist weak sources $LF_1$ and $LF_2$ with corresponding distributions $\mathcal{D}_1$ and $\mathcal{D}_2$, and base learners $f_1$ and $f_2$ with $\mathcal{L}\left(\mathcal{D}_1, f_1\right)=0$ and $\mathcal{L}\left(\mathcal{D}_2, f_2\right)=0$ such that 
    \begin{equation}
        \forall \alpha \in[0,1],\quad
        \mathcal{L}\left(\frac{1}{2}\left(\mathcal{D}_1+\mathcal{D}_2\right), \alpha f_1+(1-\alpha) f_2\right) \geq \frac{1}{2},
    \end{equation}
    where $\alpha$ is the model weights in the ensemble, $\mathcal{L}()$ quantifies the loss.
\end{prop} The proof is given in Appendix~\ref{app:proof}.

 The above proposition states that even if the base learner fits well to the weakly labeled data provided by the weak sources, the convex combination of them can still perform poorly. 
 Therefore, we consider the new ensemble form to account for the presence of weak sources. Inspired by \cite{mansour2008domain, hoffman2018algorithms, zhang2021multiple, cortes2021boosting, cortes2021discriminative}, we introduce a conditional function to compute the probability of the sample being labeled by each weak source. The major difference is that the mentioned works discuss the conditional probability in the context of domain adaption or multi-source learning and are designed for the fully supervised settings, we instead present a conditional function to account for the weak sources in the WSL settings to modulate the base learners, and naturally introduce the inter-source boosting. 

\begin{prop}
    By plugging a conditional function regarding the weak sources into the standard convex combination, there exists an ensemble in the WSL setting such that 
    \begin{equation}
        \mathcal{L}\left(\mathcal{D}_\lambda,\left(\alpha \mathbf{Q}(l \mid x) f_1(x)+(1-\alpha) \mathbf{Q}(l \mid x) f_2(x)\right)\right) = 0,
    \end{equation}
    where $\mathcal{D}_\lambda$ is a mixture of the weak sources such that $\mathcal{D}_\lambda=\sum_{k=1}^p \lambda_k \mathcal{D}_k $,
    $\sum_{k=1}^p \lambda_k=1$, and $\forall k\in[p], \lambda_k > 0$.
\end{prop}
\begin{proof}
    Consider the case of conditional function indicating the matching between the samples and the weak sources:
    \begin{equation}
        Q(LF_1|x_1) = Q(LF_2|x_2) = 1,\quad
        Q(LF_2|x_1) = Q(LF_1|x_2) = 0.
    \end{equation}
    Plug it into the standard convex combination, we get the ensemble as
    \begin{equation}
        F(x) = \sum_i\alpha Q(l|x)f_i(x).
    \end{equation}
    Then the ensemble admits no loss for the case mentioned in Prop.~\ref{prop:paradox}:
    \begin{equation}
    \begin{aligned}
        \forall \alpha \in[0,1], \quad
        &\mathcal{L}\left(\frac{1}{2}\left(\mathcal{D}_1+\mathcal{D}_2\right), F(x)\right)
        \\
        =& \frac{1}{2} \mathbb{I}(\alpha Q(1\mid x_1)f_1(x_1)+
        (1-\alpha)Q(2\mid x_1)f_2(x_1)) \\
        +& \frac{1}{2} \mathbb{I}(\alpha Q(1\mid x_2)f_1(x_2)+(1-\alpha)Q(2\mid x_2)f_2(x_2)) = 0.
    \end{aligned}
    \end{equation}
    This can be extended to a general mixture of weak sources as 
    \begin{equation}
        \begin{aligned}
            \mathcal{L}\left(\mathcal{D}_\lambda, F(x)\right) & =\left(\lambda \mathbb{I}\left(F\left(x_1\right) \leq 0\right)+(1-\lambda) \mathbb{I}\left(-F\left(x_2\right) \leq 0\right)\right) \\
            & =\lambda(\mathbb{I}(\alpha \leq 0)+(1-\lambda) \mathbb{I}((1-\alpha) \leq 0))=0.
        \end{aligned}
    \end{equation}
\end{proof}
In practice, we use an MLP to learn the conditional function on a source-index dataset $\mathcal{D}_s$ using the instance features. Given $p$ weak sources and $N_l$ unlabeled data, we can easily construct a matching matrix in the shape of $N_l \times p$ to represent the  matching results, where each entry takes a binary value from $\{0, 1\}$ to indicate if it gets matched with a specific weak source. For the $p$ weak sources, we have $\sum_{l=1}^pQ(l|x) = 1$.

In the inference stage, the advantage of the learned conditional function is more evident compared to the direct LF matching. As a straightforward substitution of the learned conditional function, we can implement an LF matching for each test sample to modulate the base learners to specific weak sources. The major problem with such a hard matching is the potential labeling conflict among multiple LFs, {as it is often quite challenging to identify the correct labeling functions based on the voting or aggregation~\cite{ratner2017snorkel,ratner2019training} approaches with
matching matrix only}.
For the learned conditional function, it can generalize better than the hard matching after training on the source-index dataset. By the output of a probability vector, it assigns different weights to the base learners according to their source relevance and thus enables better modulating.
 
\subsection{Base Learner Localization}
\label{sec:base-learner}
\begin{figure}
    \centering
    \includegraphics[scale=0.52]{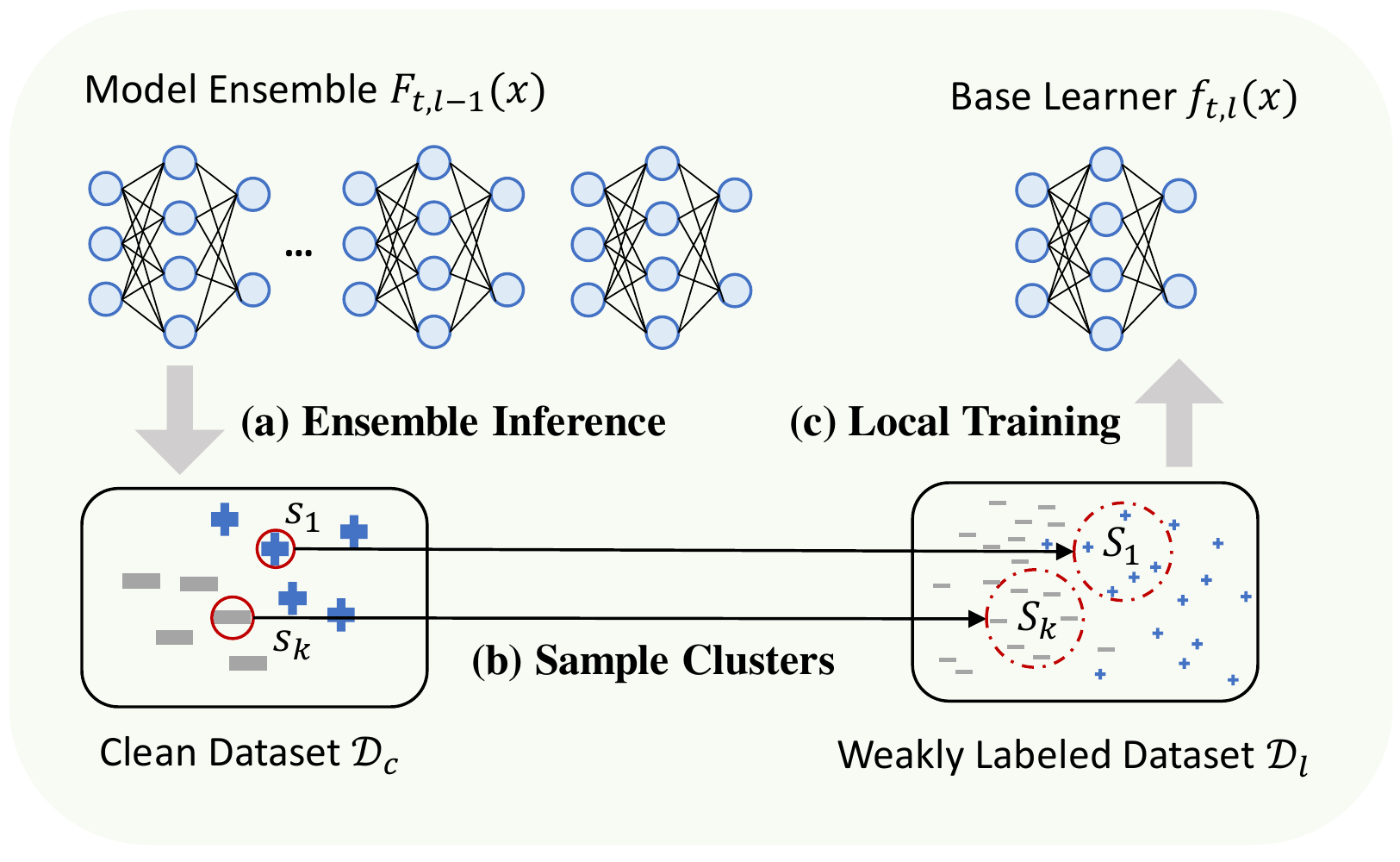}
    \caption{An illustrative example of base learner localization on a 2D plane. To localize the base learner $f_{t,l}(x)$, we first implement an ensemble inference on the clean dataset $\mathcal{D}_c$ to identify $k$ large error instances $s_1,\cdots,s_k$. Next, we sample $k$ clusters $S_1,\cdots,S_k$ on the weakly labeled dataset $\mathcal{D}_l$. Then the base learner $f_{t,l}(x)$ is trained on the local regions consisting of $S_1,\cdots,S_k$. Here we emphasize the clean dataset $\mathcal{D}_c$ is only for validation use. It guides the base learner localization but is not involved in the training.
    This figure takes $F_{t, l-1}(x)$ as the preceding ensemble, it could also be $F_{t-1, l}(x)$ when the loop over $p$ weak sources is completed. }
    \label{fig:local}
\end{figure}

The core idea of \ours\ is to allocate the base learners to local regions in the input embedding space. To illustrate this point clearly, we present a 2D visualization in Fig.~\ref{fig:local}. 
Such a locality differs from both the focus on large-error data in supervised boosting which does not explicitly account for information from feature space and the task-agnostic pre-clustering in the MoE approach. 
For \emph{traditional boosting approaches} such as AdaBoost, the learning framework iteratively computes the weighted error and each base learner is trained on a reweighted training set. 
The issues are two folds for the WSL setting: 1) First, there are not enough clean data and the noisy labels in the weakly labeled dataset cannot underpin an accurate estimation of the weighted error. 2) Second, the base learner is granted access to the entire training set. 
Although the training set has been reweighted to emphasize the accumulated errors, the 
imperfect weak labels from multiple feature regimes poses specific challenges for base learner fitting, as the boosted model is still easy to overfit the noise. 
For the \emph{MoE approach}, although it constructs clusters in the embedding space and deploys expert models to learn specialized patterns, the clusters are often assigned in a static way with instance features~\cite{yu2022coco}. 
Such a rigid scheme fails to explicitly model the training dynamics of the existing base learners, often resulting in suboptimal performance. 

To this end, we harness the small, clean validation set to guide the localization of the base learners to introduce locality for these base learners via a mapping between the error instance and the error regions. 
We first identify the large-error instances on the small clean dataset $\mathcal{D}_c$, then sample clusters on $\mathcal{D}_l$ based on the identified instances.
Denote the model ensemble at iteration $(t, l)$ as $F_{t,l}(x)$, and its prediction vector as $\mathcal{G}(F_{t,l}(x))$. We maintain an error matrix $m_{t,l}$ to record the accumulated error, which is initialized as $m_{t,l} \gets [0]^{N_c}$ at the beginning, where $N_c = |\mathcal{D}_c|$. By the ensemble inference on $\mathcal{D}_c$, we update $m_{t,l}$ as
\begin{equation}
    m_{t,l} \gets m_{t,l} + [1]^{N_c} - \mathbf{y}^T \mathcal{G}(F_{t,l}(x)),
\end{equation}
so the entries with accumulated errors get larger in the iterative process. Then we pick the top-$k$ error instances by
\begin{equation}
    \{s_{j_1}, \cdots, s_{j_k} | \argmax_{j_1,\cdots,j_k \in N_c}\sum_j^{j_1,\cdots,j_k} [m_{t,l}]_j\}.
    \label{eq:k}
\end{equation}
Based on these identified error instances, we back to $\mathcal{D}_l$ and sample $k$ clusters with a hyper-radius inversely proportional to the accumulated error:
\begin{equation}
    \forall s_j \in \{s_{j_1}, \cdots, s_{j_k}\},\quad
    S_j = \{x | \|x-s_j\| \leq c_1 d/[m_{t,l}]_j\},
    \label{eq:c1}
\end{equation}
where $c_1$ is a parameter, $d$ is the average distance of all data samples in $\mathcal{D}_l$:
\begin{equation}
    d = \frac{2}{N_l(N_l-1)}\sum_{p}\sum_{q}\|x_p - x_q\|, \quad p \neq q.
\end{equation}
In this way, we form a training set 
\begin{equation}
    \mathcal{D}_{t,l} = S_1 \cup\cdots\cup S_k,
\end{equation}
which is the \emph{local region} for the fitting of the to-be-added base learner $f_{t, l}$. We deploy the base learner $f_{t, l}$ to the local region $\mathcal{D}_{t,l}$ and fit it on $\mathcal{D}_{t,l}$ by optimizing:
\begin{equation}
  \min _{\theta} \frac{1}{\left|\mathcal{D}_{t,l}\right|} \sum\limits_{\left(\boldsymbol{x}_i,\hat{y}_{i}\right) \in\mathcal{D}_{t,l}}
  \ell_{\operatorname{CE}}\left(f_{t,l}(\boldsymbol{x}_i), \hat{y}_{i}\right),
\end{equation}
where $\hat{y}_i$ is the weak label for instance $\boldsymbol{x}_i$,  and $\ell_{\operatorname{CE}}$ is the cross entropy loss.

The above process reflects the interactive nature of \ours, \emph{i.e.,} the error instances identified on the clean dataset guide the base learner training on the weakly labeled dataset. If we use the clean dataset alone, the limited number of samples are insufficient to support the model fitting \cite{brukhim2021multiclass}. If we use the weakly labeled dataset alone, we cannot distinguish the false positive when updating the error matrix due to the presence of noisy labels. Instead, our approach first accurately identifies the large-error instances on the small clean dataset $\mathcal{D}_c$, then sample regions on the large weakly labeled dataset $\mathcal{D}_l$ based on these instances to gather sufficient supervision for the base learner training. 
Compared to the data reweighting approach in supervised boosting~\cite{freund1995desicion}, {\ours} targets only local regions in each iteration so that the base learners can be trained on more specific feature regimes, which is suitable for WSL setting because it is more difficult to directly learn from the imperfect noisy labels. 
Besides, {\ours} explicitly inherits the boosting philosophy compared to the MoE approach--the local regions evolve iteratively and adaptively to reflect the weakness of the preceding ensemble, thereby the new base learner serves as a complement.

\subsection{Estimate-then-Modify Weighting Scheme with Perturbation}\label{sec:weight}
\begin{figure}
    \centering
    \includegraphics[scale=0.52]{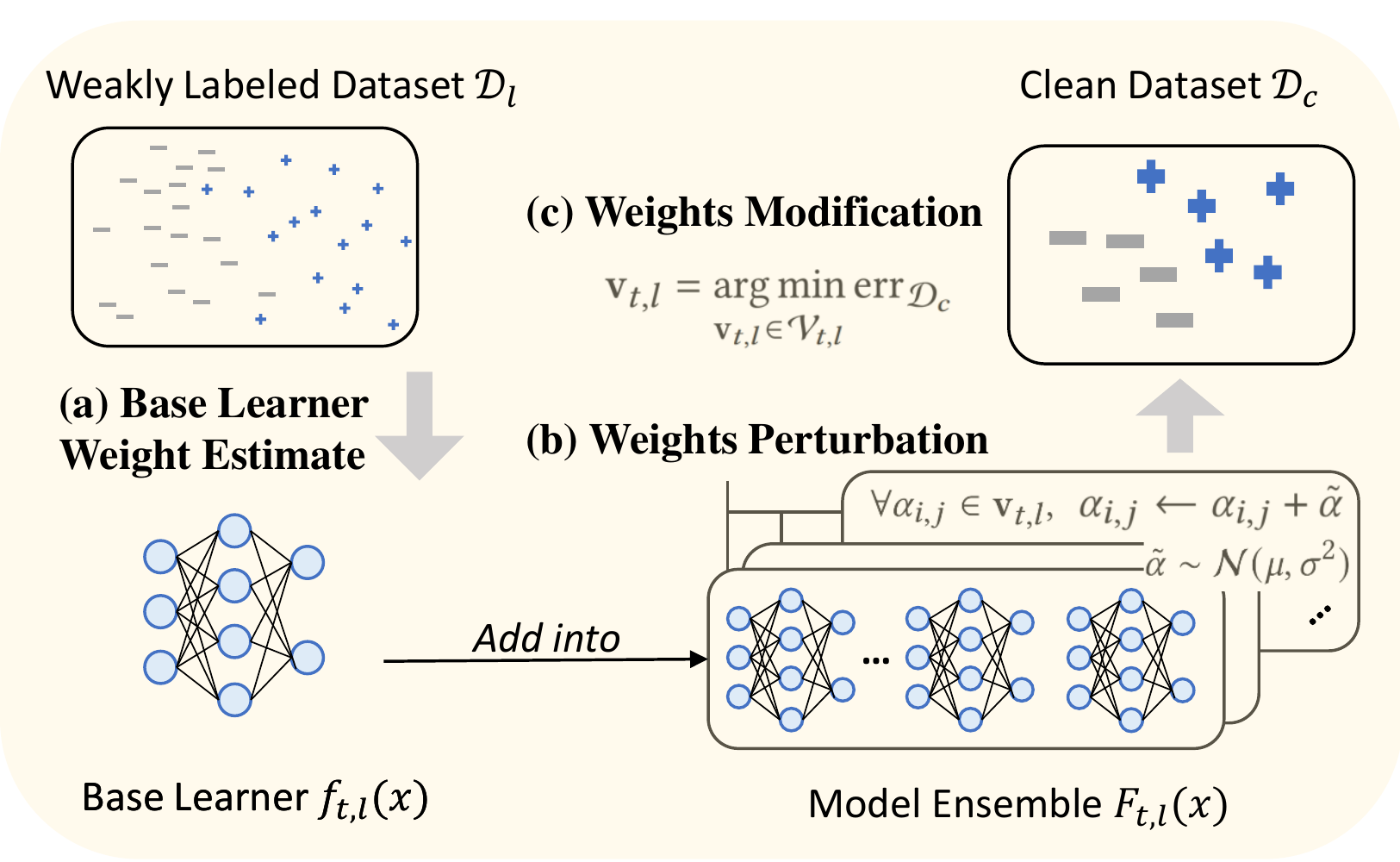}
    \caption{The illustration of Estimate-then-Modify scheme for model weight calculation. We first estimate the base learner weight $\alpha_{t,l}$ on $\mathcal{D}_l$ and form the new ensemble $F_{t,l}(x)$. Next, we generate a group of perturbed weight vectors $\mathcal{V}_{t,l}$ by adding Gaussian and normalizing each of them. Finally, we select the weight vector that achieves the lowest total error on $\mathcal{D}_c$.}
    \label{fig:weight_esco}
\end{figure}
For the weight computation, we present an adaptive design for boosting in the WSL setting, namely the \emph{estimate-then-modify} paradigm. In supervised boosting, it is easy to compute the model weights by the principle of minimizing the total error. This implementation builds upon access to a full clean dataset. In WSL, however, we only have a large weakly labeled dataset, along with a clean validation dataset with only limited examples. Therefore, we propose an estimate-then-modify paradigm for the model weights computation, in which a large number of weak labels are leveraged for the weight estimate, then the limited clean labels are used to rectify the estimated weights.

In particular, we first follow the AdaBoost procedure to estimate the weight $\alpha_{t, l}$ on the weakly labeled dataset $\mathcal{D}_l$, this starts from the data weights initialization on $\mathcal{D}_l$:
\begin{equation}
    w_i = 1/N_l, i = 1,2,\cdots,N_l,
\end{equation}
where $N_l = |\mathcal{D}_l|$. Then we calculate the weighted error rate of the current base learner $f_{t,l}$ by
\begin{equation}
  \operatorname{err}_{\mathcal{D}_l}=\sum_{i}^{N_l} w_{i} \mathbb{I}\left(y_{i} \neq f_{t,l}(x_i)\right).
\end{equation}
It follows the weight calculation of $\alpha_{t,l}$ for the base learner $f_{t,l}$:
\begin{equation}
    \label{eq:alpha}
  \alpha_{t,l}=\log \frac{1-\operatorname{err}_{\mathcal{D}_l}}{\operatorname{err}_{\mathcal{D}_l}}.
\end{equation}
This expression is based on the principle that given the to-be-added base learner, the desired weight should minimize the total error defined on the training set $\mathcal{D}_l$. Finally, we update the data weights:
\begin{equation}
  w_{i} \leftarrow w_{i} \cdot e^{\alpha_{t,l} \mathbb{I}\left(y_{i} \neq F_{t,l}\left(\boldsymbol{x}_i\right)\right)}, i=1,2, \ldots, N_l, \label{eq_weights},
\end{equation}
and normalize it such that $\sum_i w_i = 1$.
Till now, we have obtained an estimated weight of the base learner, the ensemble model can be updated by:
\begin{equation}
\begin{aligned}
    &F_{t, l}(x) =  F_{t, l-1}(x) + \alpha_{t,l}Q(l|x)f_{t,l}(x),\quad l > 1\\ \text{or}\quad 
    &F_{t, l}(x) =  F_{t-1, l}(x) + \alpha_{t,l}Q(l|x)f_{t,l}(x),\quad t > 1
\end{aligned}
\end{equation}
Note that in the above derivation, we directly use the weak labels in $\mathcal{D}_l$ for the error rate computation and the total error minimization. 
However, one key difference for weakly-supervised learning is the absence of a large amount of clean labels. 
Merely using the noisy weak labels can hardly guarantee the boosting progress due to the fact that calculating error rates  involves weak labels, which are less unreliable and can negatively affect the weight estimation.

To address this issue, we further calibrate such estimated weights using the small clean dataset $\mathcal{D}_c$ with a perturbation-based approach. 
Specifically, 
for the iteration $(t, l)$, we have the weight vector $\mathbf{v}_{t, l} = [\alpha_{1,1},\cdots,\alpha_{t,l}]$.  We then add Gaussian perturbation to the weight vector by
\begin{equation}
    \forall \alpha_{i,j} \in \mathbf{v}_{t, l},
    \quad
    \alpha_{i,j} \gets \alpha_{i,j} + \Tilde{\alpha},
    \quad
    \textit{s.t.,}\quad
    \Tilde{\alpha} \sim \mathcal{N}(\mu,\sigma^{2}),
\end{equation}
and normalize the sum of weights to 1 for getting a group of perturbed weight vectors $\mathcal{V}_{t, l} = \{\mathbf{v}_{t, l}\}^{n_p}$, where $n_p$ is the number of perturbations. 

The $\mathcal{V}_{t, l}$ enables us to validate different combinations of the base learners. We define the clean error on the small clean dataset $\mathcal{D}_c$ as
\begin{equation}
    \operatorname{err}_{\mathcal{D}_c} = \sum_{i}^{N_c} \exp (-y_i F_{t,l}\left(x_i\right)),
\end{equation}
where $y_i$ is the clean label on $\mathcal{D}_c$, and  $N_c = |\mathcal{D}_c|$. We select the weight vector with the lowest validation error~\cite{zhang2023not}
\begin{equation}
    \mathbf{v}_{t, l} = \argmin_{ \mathbf{v}_{t, l} \in \mathcal{V}_{t, l}}\operatorname{err}_{\mathcal{D}_c}
\end{equation}
as the modified base learner weights.

The weights computation for the base learners manifests another aspect of the interaction between weak labels and clean labels. There are two natural alternatives to compute the weights, either using the weak labels or the clean labels alone. We have demonstrated that using only weak labels is suboptimal due to the unreliable computation caused by the noise in the weak labels.
On the other hand, if we use clean labels alone, the base learner can easily overfit on the limited number of samples. 
Another alternative is to integrate both the weak labels and clean labels by fitting the base learner on the weakly labeled dataset, while computing the weights using the clean labels. However, we argue that the error made on the clean dataset could be caused either by the distribution shifts between the weak labels and the clean labels due to the limited and biased labeling functions, or by the base learner overfitting to the label noise. 
If we decouple the two steps of base learner fitting and weight computation, it equivalently forces the boosting process to the clean dataset only. This ignores the fact that an ensemble well-performed on a large dataset, though weakly labeled, can generalize better than an overfitted one. 
We empirically validate the above statement in Sec.~\ref{sec:exp_weight}.

\section{Experiments}
\begin{table*}[!htb]
    \centering
    \caption{Main Results. $^{*}$: Results are copied from the corresponding paper.}
    \scalebox{0.9}{
    \begin{tabular}{ccccccccc}
    \toprule
        & \textbf{IMDb (Acc.)} &  \textbf{Yelp (Acc.)}    &  \textbf{Youtube (Acc.)}  & \textbf{AGNews (Acc.)}    & \textbf{TREC (Acc.)}  & \textbf{CDR (F1)}   & \textbf{SemEval (F1)}  & \textbf{Mean} \\\midrule
       Gold & 91.58 $\pm$ 0.31 & 95.48 $\pm$ 0.53 & 97.52 $\pm$ 0.64 & 90.78 $\pm$ 0.49 & 96.24 $\pm$ 0.61 & 65.39 $\pm$ 1.18 & 95.43 $\pm$ 0.65& 90.35 \\\midrule
       Majority Voting  & 77.14 $\pm$ 0.13 & 84.91 $\pm$ 0.13 & 90.16 $\pm$ 0.13 & 63.88 $\pm$ 0.13  &  66.56 $\pm$ 1.20   & 58.89 $\pm$ 0.50  & 85.53 $\pm$ 0.13& 75.30 \\
       Weighted Voting  & 76.90 $\pm$ 0.24  &   85.45 $\pm$ 1.21    &   92.48 $\pm$ 0.16    &   83.54  $\pm$ 0.18   &   66.00 $\pm$ 2.33  & 57.53 $\pm$ 0.46  &   83.77 $\pm$ 2.93 & 77.95   \\
       Dawid-Skene& 80.25 $\pm$ 2.23 & 88.59 $\pm$ 1.25 & 92.88 $\pm$ 0.78 & 86.69 $\pm$ 0.35 & 48.40 $\pm$ 0.95 &  50.49 $\pm$ 0.48 & 71.70 $\pm$ 0.81& 74.14 \\
       Data Programming& 80.82 $\pm$ 1.29 & 82.90 $\pm$ 3.69 & 93.60 $\pm$ 0.98 & 86.55 $\pm$ 0.08& 68.64 $\pm$ 3.57 &  58.48 $\pm$ 0.73 & 83.93 $\pm$ 0.83& 79.27 \\
       MeTaL & 81.23 $\pm$ 1.23 & 88.29 $\pm$ 1.57 & 92.48 $\pm$ 0.99 & 86.82 $\pm$ 0.23 & 62.44 $\pm$ 2.96 & 58.48 $\pm$ 0.90 &  71.47 $\pm$ 0.57& 77.32 \\
       FlyingSquid& 82.26 $\pm$ 1.41 & 88.86 $\pm$ 0.92 & 91.84 $\pm$ 2.10 & 86.29 $\pm$ 0.49 & 30.96 $\pm$ 4.04 & 35.25 $\pm$ 5.75 & 31.83 $\pm$ 0.00& 63.90 \\
        EBCC$^*$ & 67.99 & 72.87 & 86.57 & 55.94 & 46.94 & 23.89 & 33.80& 55.43   \\
        FABLE$^*$ & 73.96 & 72.50 & 88.86 & 62.74 & 53.20 & 62.15 & 74.32 &69.68  \\\midrule
        Denoise & 76.22 $\pm$ 0.37 & 71.56 $\pm$ 0.00 & 76.56 $\pm$ 0.00 & 83.45 $\pm$ 0.11 & 56.20 $\pm$ 6.73 & 56.54 $\pm$ 0.37 & 80.83 $\pm$ 1.31& 71.62  \\
         WeaSEL& 84.99 $\pm$ 1.17 & 86.67 $\pm$ 2.23 & 76.80 $\pm$ 15.20 & 85.30 $\pm$ 0.72 & 62.50 $\pm$ 7.40 & 46.20 $\pm$ 12.30 & 44.30 $\pm$ 2.65& 69.54 \\ \midrule
       \rowcolor{gray!20}  \ours &  85.74 $\pm$ 1.12 & 91.50  $\pm$ 0.83 & 94.93  $\pm$ 0.79 & 88.92  $\pm$ 0.44 & 69.72 $\pm$ 1.47 & 60.29 $\pm$ 0.24 & 86.35 $\pm$ 0.57& 82.49\\ 
    \bottomrule
    \end{tabular}
    }
    \label{tab:main-res}
\end{table*}

\subsection{Experiment Setup}

\subsubsection{Datasets}
We conduct main experiments on seven public datasets from the WRENCH benchmark~\cite{zhang2021wrench}, including 1) \textbf{IMDb}~\cite{maas-EtAl:2011:ACL-HLT2011} for movie review classification; 2) \textbf{Yelp}~\cite{AGNews} for sentiment analysis; 3) \textbf{YouTube}~\cite{youtube} for spam classification; 4) \textbf{AGNews}~\cite{AGNews} for news topic classification; 5) \textbf{TREC}~\cite{trec} for web query classification; 6) \textbf{CDR}~\cite{davis2017comparative} for biomedical relation classification; and 7) \textbf{SemEval}~\cite{hendrickx2010semeval} for web text relation classification. The details for these datasets are exhibited in table \ref{tab:dataset_stats_class}.

\subsubsection{Metrics} We strictly follow the evaluation protocol in WRENCH benchmark. Specifically, for text classification, we use \emph{Accuracy} as the metric. For relation classification, we use \emph{F1} score as the metric.

\subsubsection{Baselines} We compare {\ours} with the following
baselines:
\begin{itemize}[leftmargin=0.5cm]
  \item \textbf{Majority voting}: It predicts the label of each data point using the most common prediction from LFs.
  \item \textbf{Weighted majority voting}: It extends the majority voting by reweighting the final votes using the label prior.
  \item \textbf{Dawid-Skene}~\cite{dawid1979maximum}: It estimates the accuracy of each LF by assuming a naive Bayes distribution over the LFs’ votes and models the ground truth as the latent variable.
  \item \textbf{Data Programming}~\cite{ratner2017snorkel}: It models the distributions between label and LFs as a factor graph to  reflect the dependency between any subset of random variables. It uses Gibbs sampling for maximum likelihood optimization.
  \item \textbf{MeTaL}~\cite{ratner2019training}: It models the distribution via a Markov Network and estimates the parameters via matrix completion.
  \item \textbf{FlyingSquid}~\cite{fu2020fast}: It models the distribution between LF and labels as a binary Ising model and uses a Triplet Method to recover the parameters.
  \item \textbf{EBCC}~\cite{li2019exploiting}: It is a method originally proposed for crowdsourcing, which models the relation between workers’ annotation and the ground truth label by independent assumption.
  \item \textbf{FABLE}~\cite{zhang2022leveraging}: It incorporates instance features into a statistical label model for better label aggregation.
  \item \textbf{Denoise}~\cite{ren2020denoising}: It uses an attention-based mechanism for aggregating over weak labels, and co-trains an additional neural classifier to encode the instance embeddings.
  \item \textbf{WeaSEL}~\cite{ruhling2021end}: It is an end-to-end approach for WSL, which maximizes the agreement between the neural label model and end model for better performance.
\end{itemize}

\subsubsection{Implementation Details}
We keep the number of iterations $T=5$ for all the experiments, the number of weak sources $p$ varies based on the number of LFs as shown in Table~\ref{tab:dataset_stats_class}. Specifically, we set $p = \#LF$ for IMDb, Yelp, YouTube, and AGNews. For the other datasets, it will bring a redundant ensemble if we implement inter-source boosting for all the LFs. Instead, for TREC and SemEval, we group the LFs based on the labels. For CDR, we manually divide the LFs to 6 groups (3 groups for each label). The source-index datasets are constructed accordingly, in which the index of LFs or LF groups are used as labels. To learn the conditional function, we deploy a multi-layer perceptron with 2 hidden layers, the shape of the output layer varies according to the number of weak sources for different datasets.
We use the training set shown in Table~\ref{tab:dataset_stats_class} as the weakly labeled datasets $\mathcal{D}_l$, where the labels are generated by their LFs. We use a subset of the validation set as the small clean dataset $\mathcal{D}_c$. For YouTube and SemEval, we set the $|\mathcal{D}_c|$ as $120$ and $200$. For AGNews, we set the $|\mathcal{D}_c| = 1000$. For the other datasets, we keep $|\mathcal{D}_c| = 500$. The clean data are only for the usage of weights computation and error instance identification, not involved in the base learner training. We use BERT-base~\cite{devlin2018bert} with 110M parameters as the backbone model and AdamW \cite{loshchilov2018decoupled} as the optimizer. More  details on hyperparameters are introduced in Appendix~\ref{app:impl}.

\vspace{-5pt}
\subsection{Main Results}
The results in Table~\ref{tab:main-res} compare the performance of our method {\ours} with the baselines. On all datasets, {\ours} consistently outperforms these strong  baselines, with performance improvements ranging from 1.08\% to 3.48\% compared to the strongest baselines. 
On average across the seven datasets, {\ours} reaches 91.3\% of the performance obtained with fully-supervised learning, significantly reducing the gap between weakly-supervised learning and fully-supervised learning.

The performance of our method is significantly better compared to the voting methods. This demonstrates that our conditional function can effectively boost the performance of the base learners by computing the conditional probability of weak sources given an input sample. Unlike the voting methods, which simply average or weight-average the predictions from different sources, the conditional function allows the base learners to emphasize on most reliable labeling functions when combined.

When comparing with the label aggregation methods, the significant improvement demonstrates the advantage of boosting methods over the single end model. 
Take the strongest baseline (Data Programming) for instance.
On the IMDb and Yelp datasets, {\ours} achieves a performance gain of 4.92\% and 8.60\%, respectively.
Although label aggregation methods can model the distributions between labels and weak sources, the single end model trained on the entire weakly labeled dataset fails to obtain a locality to enhance itself.
Instead, we introduce locality to the base learners while retaining the boosting nature, so the ensemble model in \ours\ can self-strengthen by adding complementary base learners iteratively.


\begin{figure*}[!htb]
  \centering
  \begin{subfigure}{0.33\textwidth}
    \includegraphics[width=\textwidth]{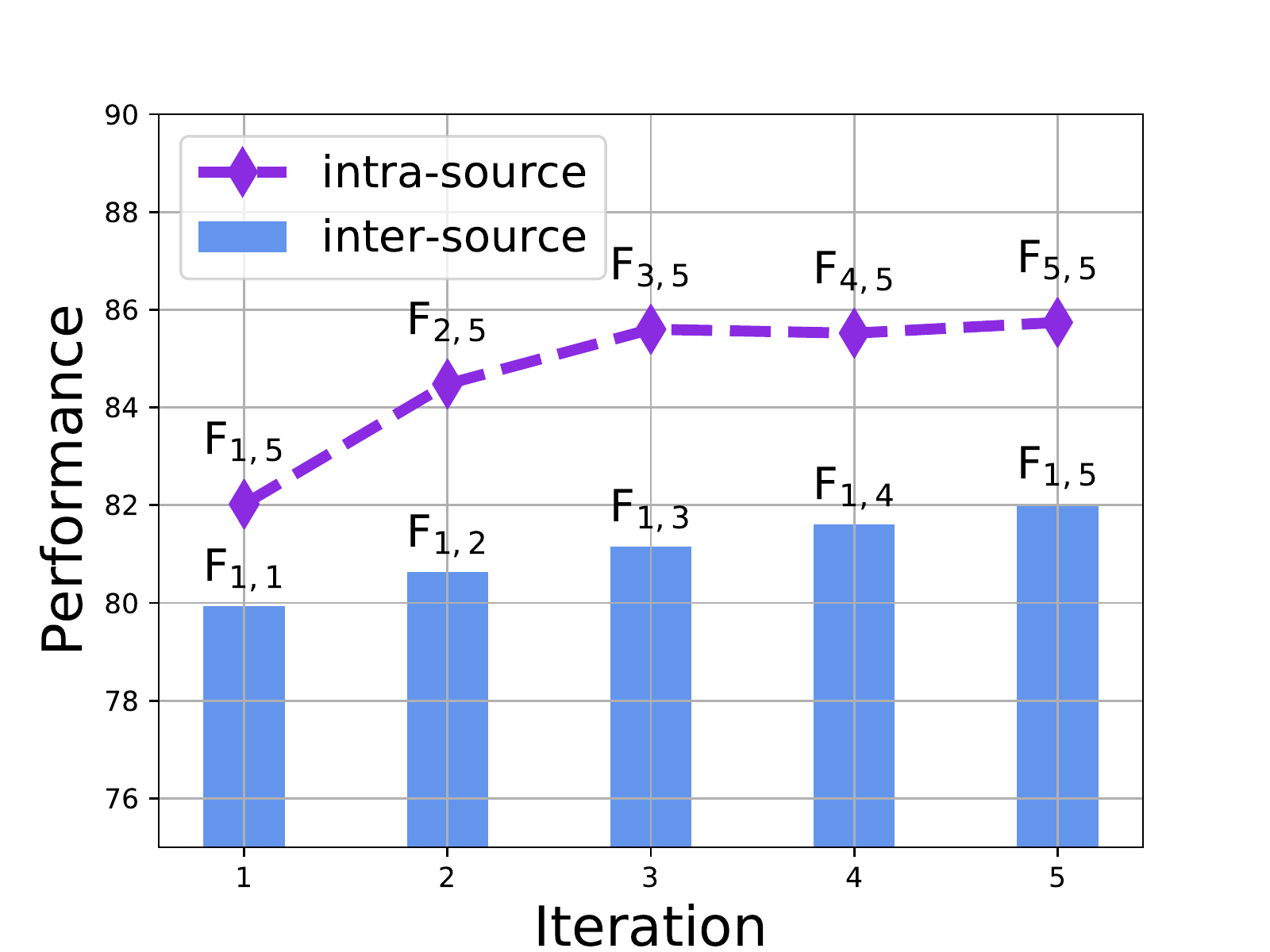}
    \caption{IMDb}
    \label{fig:first}
  \end{subfigure}
  \hfill
  \begin{subfigure}{0.33\textwidth}
    \includegraphics[width=\textwidth]{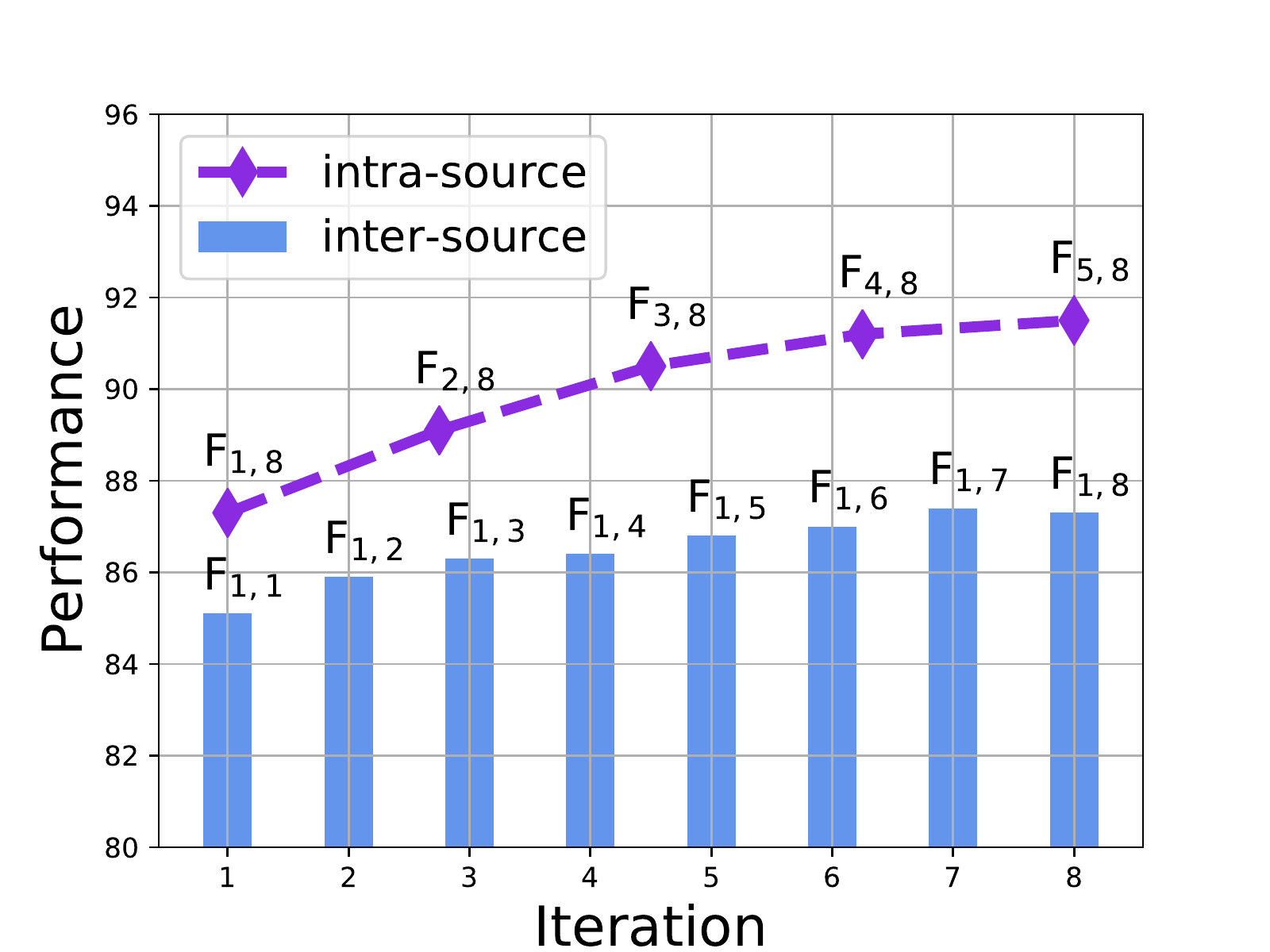}
    \caption{Yelp}
    \label{fig:second}
  \end{subfigure}
  \hfill
  \begin{subfigure}{0.33\textwidth}
    \includegraphics[width=\textwidth]{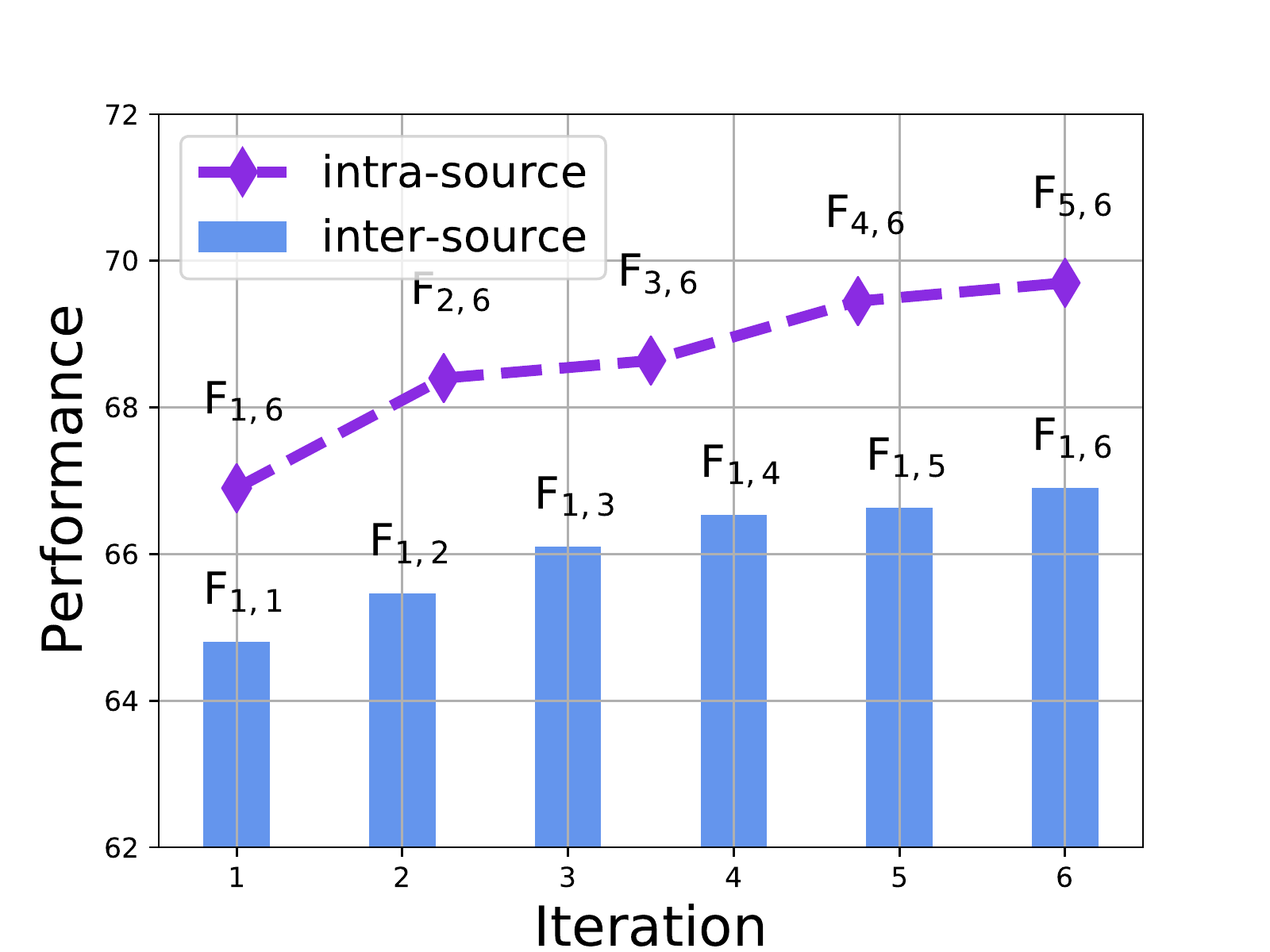}
    \caption{TREC}
    \label{fig:third}
  \end{subfigure}
\vspace{-4ex}
  \caption{Iterative performance of the two-dimension boosting. For the intra-source boosting, we show its performance change in the first $t=1$ loop. For the inter-source boosting, we show the ensemble at the end of each outer iteration,  where $l=p$. }
  \label{fig:2d_boosting}
\end{figure*}

\subsection{Two-dimension Boosting}

Fig.~\ref{fig:2d_boosting} shows the iterative results in the two-dimension boosting process.
For inter-source boosting, we plot its performance change at iteration $(t,l)$ such that $t=1, l\in[p]$.
For intra-source boosting, we plot its performance change at iteration $(t,l)$ such that $l=p, t\in[T]$.
We observe a consistent improvement in performance through inter-source boosting, indicating that the base learners can complement each other and build upon the models trained from other weak sources.
In the early stages of inter-source boosting, there is a significant improvement in performance, demonstrating that the weak regions of the previous ensemble model are effectively learned by the following base learners. However, the performance gains become relatively modest towards the end, as the ensemble model has already combined sufficient base learners, and the remaining error regions may not be well learned due to the limitations of weak supervision, even with the addition of more base learners.


\subsection{Study of the Conditional Function}
In this set of experiments, we study the effect of the conditional function $Q(l|x)$ as described in Sec.~\ref{sec:cond_func}.
We considered two baselines, one where the conditional function is disabled, and another where the conditional function is replaced with labeling function (LF) matching.
The first baseline  simply ensembles the base learners without distinguishing between intra-source and inter-source boosting.
For the second baseline, the direct LF matching is a natural alternative to the learned conditional function. Specifically, to represent the conditional probability of weak sources given a data sample, we replace the $Q(l|x)$ in the \ours\ framework with the entries in a normalized LF matching vector.
\begin{table}[!htb]
    \centering
    \caption{Study of the conditional function}
    \begin{tabular}{cccc}\toprule
         &\cellcolor{gray!20} \ours & w/o. cond. func.  & w/ LF matching  \\\midrule
            \textbf{IMDb} & \cellcolor{gray!20} 85.74 $\pm$ 1.12 & 80.37 $\pm$ 1.23& 84.92 $\pm$ 1.63\\ 
          \textbf{Yelp} &\cellcolor{gray!20} 91.50  $\pm$ 0.83 & 88.07 $\pm$ 2.45 &88.89  $\pm$ 2.10 \\ 
          \textbf{Youtube}  &\cellcolor{gray!20} 94.93  $\pm$ 0.79 & 93.31  $\pm$ 1.14 & 93.23 $\pm$  1.27\\
          \textbf{AGNews}  &\cellcolor{gray!20} 88.92  $\pm$ 0.44 & 87.19 $\pm$ 0.54 & 88.45 $\pm$ 0.24 \\
          \textbf{TREC}  &\cellcolor{gray!20} 69.72 $\pm$ 1.47 & 67.78 $\pm$ 2.19 & 68.42 $\pm$ 1.31\\
          \textbf{CDR}   &\cellcolor{gray!20} 60.29 $\pm$ 0.24 & 56.19 $\pm$ 2.56 & 58.79 $\pm$ 1.43\\
          \textbf{SemEval}  & \cellcolor{gray!20} 86.35 $\pm$ 0.57& 85.64 $\pm$ 0.76 & 86.23 $\pm$ 0.48\\\midrule
          \textbf{Mean} &\cellcolor{gray!20} 82.49 & 79.79 & 81.28 \\\bottomrule
    \end{tabular}
    \label{tab:cond_func}
\end{table}

Table~\ref{tab:cond_func} compares the learning framework variants and demonstrates the advantages of \ours.
It is consistently better on all seven datasets, with average performance gains of 2.70\% and 1.21\%.
The comparison with the baseline without $Q(l|x)$ highlights the superiority of introducing a conditional function to account for weak sources.
Disabling $Q(l|x)$ results in a standard convex combination, which is not suitable for the WSL scenario.
The baseline using direct LF matching instead of the conditional function shows improved performance over the original baseline, supporting the need to modulate base learners for weak sources.
However, it still falls behind \ours\ by 1.21\%, as LF labeling conflicts undermine the proper assignment of base learners to weak sources.
The conditional function learned on the source-index dataset can generalize better during inference, benefiting the entire ensemble, compared to hard LF matching.

\subsection{Study of the Estimate-then-Modify Scheme}\label{sec:exp_weight}
In this set of experiments, we examine the benefits of our estimate-then-modify scheme for calculating weights in the WSL setting and the importance of the interaction between weak labels and clean labels.
We compare with two baselines, both of which are variants of \ours\ that use different weight calculation methods.
The first baseline uses the AdaBoost approach to compute base learner weights on the weakly labeled dataset, while the second integrates both weak and clean labels but only uses weak labels for base learner training and calculates weights solely on the clean dataset.

\begin{table}[!htb]
    \centering
    \caption{Study of the estimate-then-modify weighting}
    \begin{tabular}{cccc}\toprule
         &\cellcolor{gray!20} \ours & only weak labels  & integrated mode  \\ \midrule
          \textbf{IMDb} & \cellcolor{gray!20} 85.74 $\pm$ 1.12 & 79.55 $\pm$ 1.78 & 80.86 $\pm$ 1.53	 \\   
          \textbf{Yelp} &\cellcolor{gray!20} 91.50  $\pm$ 0.83 &87.78 $\pm$ 2.46 & 88.34 $\pm$ 1.20\\   
          \textbf{Youtube}  &\cellcolor{gray!20} 94.93  $\pm$ 0.79 & 92.81 $\pm$ 0.95& 93.07 $\pm$ 0.89\\ 
          \textbf{AGNews}  &\cellcolor{gray!20} 88.92  $\pm$ 0.44 & 85.60 $\pm$ 1.16 & 88.12 $\pm$ 0.42	\\ 
          \textbf{TREC}  &\cellcolor{gray!20} 69.72 $\pm$ 1.47 & 57.93 $\pm$ 4.85 &	66.30 $\pm$ 2.66\\ 
          \textbf{CDR}   &\cellcolor{gray!20} 60.29 $\pm$ 0.24 & 54.62 $\pm$ 1.69	& 58.75 $\pm$ 0.34\\ 
          \textbf{SemEval}  &\cellcolor{gray!20} 86.35 $\pm$ 0.57 & 82.86 $\pm$ 1.03 & 84.76 $\pm$ 0.71\\\midrule
          \textbf{Mean} &\cellcolor{gray!20} 82.49& 77.31 & 80.03 \\\bottomrule
    \end{tabular}
    \label{tab:interaction}
\end{table}
Table~\ref{tab:interaction} shows that \ours\ outperforms the two variants by significant margins. \ours\ has a performance improvement of 5.18\% compared to the baseline using only weak labels, supporting our hypothesis that weak labels are unreliable for weight computation.
Despite the integration of both weak and clean labels in the second baseline, \ours\ still leads by 2.46\%.
We believe that the disconnected training of base learners and weight calculation weakens the learning framework, as it cannot provide an appropriate combination of base learners when the weakly labeled data deviates from the clean data.
On the other hand, \ours\ leverages the estimate-then-modify paradigm to consider both weak and clean labels, which results in a more generalizable model due to the interaction between the two types of labels.

\section{Conclusion}

We presented a novel iterative and adaptive learning framework, \ours\, to boost the ensemble model in the setting of weakly supervised learning. While preserving the key concepts of traditional boosting methods, we introduced locality to base learners and designed an interaction between weak labels and clean labels to adapt to the WSL setting. The adaptations included the use of local base learners, the incorporation of a conditional function to account for weak sources, and the application of the estimate-then-modify scheme for weight computation. Specifically, we localized the base learners to the iteratively updated error regions in the embedding space, thereby overcoming the issues of weight domination in vanilla boosting for WSL and the challenges of hard clustering in the MoE approach. To handle weak sources such as labeling functions in WSL settings, we designed a conditional function to modulate the base learners towards weak sources with high relevance. This addressed the issue of poor ensemble performance in the standard convex combination approach in WSL settings. Finally, we proposed an estimate-then-modify scheme for weight computation. Our comprehensive empirical study on seven datasets demonstrated the effectiveness and advantages of \ours~ compared to standard ensemble methods and WSL baselines.
\begin{acks}
  This work was supported by NSF IIS-2008334, IIS-2106961, CAREER IIS-2144338, as well as a generous gift from The Home Depot.
\end{acks}



\balance
\bibliographystyle{unsrtnat}
\bibliography{sample-base}

\begin{thebibliography}{62}
\providecommand{\natexlab}[1]{#1}
\providecommand{\url}[1]{\texttt{#1}}
\expandafter\ifx\csname urlstyle\endcsname\relax
  \providecommand{\doi}[1]{doi: #1}\else
  \providecommand{\doi}{doi: \begingroup \urlstyle{rm}\Url}\fi

\bibitem[Awasthi et~al.(2020)Awasthi, Ghosh, Goyal, and
  Sarawagi]{awasthi2020learning}
Abhijeet Awasthi, Sabyasachi Ghosh, Rasna Goyal, and Sunita Sarawagi.
\newblock Learning from rules generalizing labeled exemplars.
\newblock In \emph{International Conference on Learning Representations}, 2020.

\bibitem[Safranchik et~al.(2020)Safranchik, Luo, and
  Bach]{safranchik2020weakly}
Esteban Safranchik, Shiying Luo, and Stephen Bach.
\newblock Weakly supervised sequence tagging from noisy rules.
\newblock In \emph{Proceedings of the AAAI Conference on Artificial
  Intelligence}, volume~34, pages 5570--5578, 2020.

\bibitem[Zhang et~al.(2022{\natexlab{a}})Zhang, West, Cui, and
  Zhang]{zhang2022adaptive}
Rongzhi Zhang, Rebecca West, Xiquan Cui, and Chao Zhang.
\newblock Adaptive multi-view rule discovery for weakly-supervised compatible
  products prediction.
\newblock In \emph{KDD}, pages 4521--4529, 2022{\natexlab{a}}.

\bibitem[Zhang et~al.(2021{\natexlab{a}})Zhang, Yu, Li, Wang, Yang, Yang, and
  Ratner]{zhang2021wrench}
Jieyu Zhang, Yue Yu, Yinghao Li, Yujing Wang, Yaming Yang, Mao Yang, and
  Alexander Ratner.
\newblock {WRENCH}: A comprehensive benchmark for weak supervision.
\newblock In \emph{NeurIPS}, 2021{\natexlab{a}}.

\bibitem[Freund and Schapire(1995)]{freund1995desicion}
Yoav Freund and Robert~E Schapire.
\newblock A desicion-theoretic generalization of on-line learning and an
  application to boosting.
\newblock In \emph{Computational Learning Theory: Second European Conference},
  pages 23--37. Springer, 1995.

\bibitem[Freund et~al.(1996)Freund, Schapire, et~al.]{freund1996experiments}
Yoav Freund, Robert~E Schapire, et~al.
\newblock Experiments with a new boosting algorithm.
\newblock In \emph{icml}, volume~96, pages 148--156. Citeseer, 1996.

\bibitem[Mason et~al.(1999)Mason, Baxter, Bartlett, and
  Frean]{mason1999boosting}
Llew Mason, Jonathan Baxter, Peter Bartlett, and Marcus Frean.
\newblock Boosting algorithms as gradient descent.
\newblock \emph{Advances in neural information processing systems}, 12, 1999.

\bibitem[Chen and Guestrin(2016)]{chen2016xgboost}
Tianqi Chen and Carlos Guestrin.
\newblock Xgboost: A scalable tree boosting system.
\newblock In \emph{Proceedings of the 22nd acm sigkdd international conference
  on knowledge discovery and data mining}, pages 785--794, 2016.

\bibitem[R{\"a}tsch et~al.(2005)R{\"a}tsch, Warmuth, and
  Shawe-Taylor]{ratsch2005efficient}
Gunnar R{\"a}tsch, Manfred~K Warmuth, and John Shawe-Taylor.
\newblock Efficient margin maximizing with boosting.
\newblock \emph{Journal of Machine Learning Research}, 6\penalty0 (12), 2005.

\bibitem[Jacobs et~al.(1991)Jacobs, Jordan, Nowlan, and
  Hinton]{jacobs1991adaptive}
Robert~A Jacobs, Michael~I Jordan, Steven~J Nowlan, and Geoffrey~E Hinton.
\newblock Adaptive mixtures of local experts.
\newblock \emph{Neural computation}, 3\penalty0 (1):\penalty0 79--87, 1991.

\bibitem[Tsai et~al.(2021)Tsai, Li, and Zhu]{tsai2021mice}
Tsung~Wei Tsai, Chongxuan Li, and Jun Zhu.
\newblock Mice: Mixture of contrastive experts for unsupervised image
  clustering.
\newblock In \emph{International conference on learning representations}, 2021.

\bibitem[Zhuang et~al.(2022)Zhuang, Li, Cheung, Yu, Mou, Chen, Song, and
  Zhang]{zhuang2022resel}
Yuchen Zhuang, Yinghao Li, Jerry~Junyang Cheung, Yue Yu, Yingjun Mou, Xiang
  Chen, Le~Song, and Chao Zhang.
\newblock Resel: N-ary relation extraction from scientific text and tables by
  learning to retrieve and select.
\newblock \emph{arXiv preprint arXiv:2210.14427}, 2022.

\bibitem[Xu et~al.(2023)Xu, Yu, Ho, and Yang]{xu2023weakly}
Ran Xu, Yue Yu, Joyce~C Ho, and Carl Yang.
\newblock Weakly-supervised scientific document classification via
  retrieval-augmented multi-stage training.
\newblock In \emph{the 46th International ACM SIGIR Conference on Research and
  Development in Information Retrieval}, 2023.

\bibitem[Hoffmann et~al.(2011)Hoffmann, Zhang, Ling, Zettlemoyer, and
  Weld]{hoffmann2011knowledge}
Raphael Hoffmann, Congle Zhang, Xiao Ling, Luke Zettlemoyer, and Daniel~S Weld.
\newblock Knowledge-based weak supervision for information extraction of
  overlapping relations.
\newblock In \emph{Proceedings of the 49th annual meeting of the association
  for computational linguistics: human language technologies}, pages 541--550,
  2011.

\bibitem[Smith et~al.(2022)Smith, Fries, Hancock, and Bach]{smith2022language}
Ryan Smith, Jason~A Fries, Braden Hancock, and Stephen~H Bach.
\newblock Language models in the loop: Incorporating prompting into weak
  supervision.
\newblock \emph{arXiv preprint arXiv:2205.02318}, 2022.

\bibitem[Yu et~al.(2023)Yu, Zhuang, Zhang, Meng, Shen, and Zhang]{yu2023zero}
Yue Yu, Yuchen Zhuang, Rongzhi Zhang, Yu~Meng, Jiaming Shen, and Chao Zhang.
\newblock Regen: Zero-shot text classification via training data generation
  with progressive dense retrieval.
\newblock In \emph{Findings of ACL}, 2023.

\bibitem[Zhang et~al.(2022{\natexlab{b}})Zhang, Yu, Shetty, Song, and
  Zhang]{zhang2022prompt}
Rongzhi Zhang, Yue Yu, Pranav Shetty, Le~Song, and Chao Zhang.
\newblock Prboost: Prompt-based rule discovery and boosting for interactive
  weakly-supervised learning.
\newblock In \emph{Proceedings of the 60th Annual Meeting of the Association
  for Computational Linguistics (Volume 1: Long Papers)}, pages 745--758,
  2022{\natexlab{b}}.

\bibitem[Zhang et~al.(2022{\natexlab{c}})Zhang, Hsieh, Yu, Zhang, and
  Ratner]{zhang2022survey}
Jieyu Zhang, Cheng-Yu Hsieh, Yue Yu, Chao Zhang, and Alexander Ratner.
\newblock A survey on programmatic weak supervision.
\newblock \emph{arXiv preprint arXiv:2202.05433}, 2022{\natexlab{c}}.

\bibitem[Ratner et~al.(2017)Ratner, Bach, Ehrenberg, Fries, Wu, and
  R{\'e}]{ratner2017snorkel}
Alexander Ratner, Stephen~H Bach, Henry Ehrenberg, Jason Fries, Sen Wu, and
  Christopher R{\'e}.
\newblock Snorkel: Rapid training data creation with weak supervision.
\newblock In \emph{VLDB}, volume~11, page 269, 2017.

\bibitem[Ratner et~al.(2019)Ratner, Hancock, Dunnmon, Sala, Pandey, and
  R{\'e}]{ratner2019training}
Alexander Ratner, Braden Hancock, Jared Dunnmon, Frederic Sala, Shreyash
  Pandey, and Christopher R{\'e}.
\newblock Training complex models with multi-task weak supervision.
\newblock In \emph{Proceedings of the AAAI Conference on Artificial
  Intelligence}, pages 4763--4771, 2019.

\bibitem[Fu et~al.(2020)Fu, Chen, Sala, Hooper, Fatahalian, and
  R\'e]{fu2020fast}
Daniel~Y. Fu, Mayee~F. Chen, Frederic Sala, Sarah~M. Hooper, Kayvon Fatahalian,
  and Christopher R\'e.
\newblock Fast and three-rious: Speeding up weak supervision with triplet
  methods.
\newblock In \emph{Proceedings of the 37th International Conference on Machine
  Learning}, 2020.

\bibitem[Zhang et~al.(2022{\natexlab{d}})Zhang, Song, and
  Ratner]{zhang2022leveraging}
Jieyu Zhang, Linxin Song, and Alexander Ratner.
\newblock Leveraging instance features for label aggregation in programmatic
  weak supervision.
\newblock \emph{arXiv preprint arXiv:2210.02724}, 2022{\natexlab{d}}.

\bibitem[Stephan et~al.(2022)Stephan, Kougia, and
  Roth]{stephan-etal-2022-sepll}
Andreas Stephan, Vasiliki Kougia, and Benjamin Roth.
\newblock {S}ep{LL}: Separating latent class labels from weak supervision
  noise.
\newblock In \emph{Findings of EMNLP}, December 2022.

\bibitem[Yu et~al.(2021)Yu, Zuo, Jiang, Ren, Zhao, and Zhang]{yu2021fine}
Yue Yu, Simiao Zuo, Haoming Jiang, Wendi Ren, Tuo Zhao, and Chao Zhang.
\newblock Fine-tuning pre-trained language model with weak supervision: A
  contrastive-regularized self-training approach.
\newblock In \emph{NAACL}, pages 1063--1077, 2021.

\bibitem[Wu et~al.(2020)Wu, Kuang, Zhang, Liu, Sun, Xiao, Zhuang, Si, and
  Wu]{wu2020biased}
Yiquan Wu, Kun Kuang, Yating Zhang, Xiaozhong Liu, Changlong Sun, Jun Xiao,
  Yueting Zhuang, Luo Si, and Fei Wu.
\newblock De-biased court’s view generation with causality.
\newblock In \emph{Proceedings of the 2020 Conference on Empirical Methods in
  Natural Language Processing (EMNLP)}, pages 763--780, 2020.

\bibitem[Li and Zhang(2021)]{li2021improved}
Dongyue Li and Hongyang Zhang.
\newblock Improved regularization and robustness for fine-tuning in neural
  networks.
\newblock In \emph{NeurIPS}, 2021.

\bibitem[R{\"u}hling~Cachay et~al.(2021)R{\"u}hling~Cachay, Boecking, and
  Dubrawski]{ruhling2021end}
Salva R{\"u}hling~Cachay, Benedikt Boecking, and Artur Dubrawski.
\newblock End-to-end weak supervision.
\newblock \emph{Advances in Neural Information Processing Systems},
  34:\penalty0 1845--1857, 2021.

\bibitem[Ren et~al.(2020)Ren, Li, Su, Kartchner, Mitchell, and
  Zhang]{ren2020denoising}
Wendi Ren, Yinghao Li, Hanting Su, David Kartchner, Cassie Mitchell, and Chao
  Zhang.
\newblock Denoising multi-source weak supervision for neural text
  classification.
\newblock In \emph{Findings of the Association for Computational Linguistics:
  EMNLP 2020}, pages 3739--3754, 2020.

\bibitem[Sam and Kolter(2022)]{sam2022losses}
Dylan Sam and J~Zico Kolter.
\newblock Losses over labels: Weakly supervised learning via direct loss
  construction.
\newblock \emph{arXiv preprint arXiv:2212.06921}, 2022.

\bibitem[Shin et~al.(2022)Shin, Li, Vishwakarma, Roberts, and
  Sala]{shin2022universalizing}
Changho Shin, Winfred Li, Harit Vishwakarma, Nicholas~Carl Roberts, and
  Frederic Sala.
\newblock Universalizing weak supervision.
\newblock In \emph{International Conference on Learning Representations}, 2022.

\bibitem[Wu et~al.(2022)Wu, Liu, Lu, Zhang, Feng, Sun, Wu, and
  Kuang]{wu2022towards}
Yiquan Wu, Yifei Liu, Weiming Lu, Yating Zhang, Jun Feng, Changlong Sun, Fei
  Wu, and Kun Kuang.
\newblock Towards interactivity and interpretability: A rationale-based legal
  judgment prediction framework.
\newblock In \emph{Proceedings of the 2022 Conference on Empirical Methods in
  Natural Language Processing}, pages 4787--4799, 2022.

\bibitem[Zhang et~al.(2020)Zhang, Yu, and Zhang]{zhang2020seqmix}
Rongzhi Zhang, Yue Yu, and Chao Zhang.
\newblock Seqmix: Augmenting active sequence labeling via sequence mixup.
\newblock In \emph{EMNLP}, pages 8566--8579, 2020.

\bibitem[Biegel et~al.(2021)Biegel, El-Khatib, Oliveira, Baak, and
  Aben]{biegel2021active}
Samantha Biegel, Rafah El-Khatib, Luiz Otavio Vilas~Boas Oliveira, Max Baak,
  and Nanne Aben.
\newblock Active weasul: Improving weak supervision with active learning.
\newblock \emph{arXiv preprint arXiv:2104.14847}, 2021.

\bibitem[Yu et~al.(2022{\natexlab{a}})Yu, Kong, Zhang, Zhang, and
  Zhang]{yu2022actune}
Yue Yu, Lingkai Kong, Jieyu Zhang, Rongzhi Zhang, and Chao Zhang.
\newblock Actune: Uncertainty-based active self-training for active fine-tuning
  of pretrained language models.
\newblock In \emph{NAACL}, pages 1422--1436, 2022{\natexlab{a}}.

\bibitem[Yu et~al.(2022{\natexlab{b}})Yu, Zhang, Xu, Zhang, Shen, and
  Zhang]{yu2022cold}
Yue Yu, Rongzhi Zhang, Ran Xu, Jieyu Zhang, Jiaming Shen, and Chao Zhang.
\newblock Cold-start data selection for few-shot language model fine-tuning: A
  prompt-based uncertainty propagation approach.
\newblock \emph{arXiv preprint arXiv:2209.06995}, 2022{\natexlab{b}}.

\bibitem[Guan et~al.(2018)Guan, Gulshan, Dai, and Hinton]{guan2018said}
Melody Guan, Varun Gulshan, Andrew Dai, and Geoffrey Hinton.
\newblock Who said what: Modeling individual labelers improves classification.
\newblock In \emph{Proceedings of the AAAI conference on artificial
  intelligence}, 2018.

\bibitem[Zhao et~al.(2022)Zhao, Zheng, Mukherjee, McCann, and
  Awadallah]{zhao2022admoe}
Yue Zhao, Guoqing Zheng, Subhabrata Mukherjee, Robert McCann, and Ahmed
  Awadallah.
\newblock Admoe: Anomaly detection with mixture-of-experts from noisy labels.
\newblock \emph{arXiv preprint arXiv:2208.11290}, 2022.

\bibitem[Breiman(1999)]{breiman1999prediction}
Leo Breiman.
\newblock Prediction games and arcing algorithms.
\newblock \emph{Neural computation}, 11\penalty0 (7):\penalty0 1493--1517,
  1999.

\bibitem[Friedman(2001)]{friedman2001greedy}
Jerome~H Friedman.
\newblock Greedy function approximation: a gradient boosting machine.
\newblock \emph{Annals of statistics}, pages 1189--1232, 2001.

\bibitem[Brukhim et~al.(2021)Brukhim, Hazan, Moran, Mukherjee, and
  Schapire]{brukhim2021multiclass}
Nataly Brukhim, Elad Hazan, Shay Moran, Indraneel Mukherjee, and Robert~E
  Schapire.
\newblock Multiclass boosting and the cost of weak learning.
\newblock \emph{Advances in Neural Information Processing Systems},
  34:\penalty0 3057--3067, 2021.

\bibitem[Cortes et~al.(2021{\natexlab{a}})Cortes, Mohri, Storcheus, and
  Suresh]{cortes2021boosting}
Corinna Cortes, Mehryar Mohri, Dmitry Storcheus, and Ananda~Theertha Suresh.
\newblock Boosting with multiple sources.
\newblock \emph{Advances in Neural Information Processing Systems},
  34:\penalty0 17373--17387, 2021{\natexlab{a}}.

\bibitem[Zhang et~al.(2022{\natexlab{e}})Zhang, Zhang, Courville, Bengio,
  Ravikumar, and Suggala]{zhang2022building}
Dinghuai Zhang, Hongyang Zhang, Aaron Courville, Yoshua Bengio, Pradeep
  Ravikumar, and Arun~Sai Suggala.
\newblock Building robust ensembles via margin boosting.
\newblock In \emph{International Conference on Machine Learning}, pages
  26669--26692. PMLR, 2022{\natexlab{e}}.

\bibitem[Huang et~al.(2018)Huang, Ash, Langford, and
  Schapire]{huang2018learning}
Furong Huang, Jordan Ash, John Langford, and Robert Schapire.
\newblock Learning deep resnet blocks sequentially using boosting theory.
\newblock In \emph{International Conference on Machine Learning}, pages
  2058--2067. PMLR, 2018.

\bibitem[Nitanda and Suzuki(2018)]{nitanda2018functional}
Atsushi Nitanda and Taiji Suzuki.
\newblock Functional gradient boosting based on residual network perception.
\newblock In \emph{International Conference on Machine Learning}, pages
  3819--3828. PMLR, 2018.

\bibitem[Suggala et~al.(2020)Suggala, Liu, and
  Ravikumar]{suggala2020generalized}
Arun Suggala, Bingbin Liu, and Pradeep Ravikumar.
\newblock Generalized boosting.
\newblock \emph{Advances in neural information processing systems},
  33:\penalty0 8787--8797, 2020.

\bibitem[Taherkhani et~al.(2020)Taherkhani, Cosma, and
  McGinnity]{taherkhani2020adaboost}
Aboozar Taherkhani, Georgina Cosma, and T~Martin McGinnity.
\newblock Adaboost-cnn: An adaptive boosting algorithm for convolutional neural
  networks to classify multi-class imbalanced datasets using transfer learning.
\newblock \emph{Neurocomputing}, 404:\penalty0 351--366, 2020.

\bibitem[Mansour et~al.(2008)Mansour, Mohri, and
  Rostamizadeh]{mansour2008domain}
Yishay Mansour, Mehryar Mohri, and Afshin Rostamizadeh.
\newblock Domain adaptation with multiple sources.
\newblock \emph{Advances in neural information processing systems}, 21, 2008.

\bibitem[Hoffman et~al.(2018)Hoffman, Mohri, and Zhang]{hoffman2018algorithms}
Judy Hoffman, Mehryar Mohri, and Ningshan Zhang.
\newblock Algorithms and theory for multiple-source adaptation.
\newblock \emph{Advances in Neural Information Processing Systems}, 31, 2018.

\bibitem[Cortes et~al.(2021{\natexlab{b}})Cortes, Mohri, Suresh, and
  Zhang]{cortes2021discriminative}
Corinna Cortes, Mehryar Mohri, Ananda~Theertha Suresh, and Ningshan Zhang.
\newblock A discriminative technique for multiple-source adaptation.
\newblock In \emph{International Conference on Machine Learning}, pages
  2132--2143. PMLR, 2021{\natexlab{b}}.

\bibitem[Zhang et~al.(2021{\natexlab{b}})Zhang, Mohri, and
  Hoffman]{zhang2021multiple}
Ningshan Zhang, Mehryar Mohri, and Judy Hoffman.
\newblock Multiple-source adaptation theory and algorithms.
\newblock \emph{Annals of Mathematics and Artificial Intelligence},
  89:\penalty0 237--270, 2021{\natexlab{b}}.

\bibitem[Yu et~al.(2022{\natexlab{c}})Yu, Xiong, Sun, Zhang, and
  Overwijk]{yu2022coco}
Yue Yu, Chenyan Xiong, Si~Sun, Chao Zhang, and Arnold Overwijk.
\newblock Coco-dr: Combating distribution shifts in zero-shot dense retrieval
  with contrastive and distributionally robust learning.
\newblock \emph{arXiv preprint arXiv:2210.15212}, 2022{\natexlab{c}}.

\bibitem[Zhang et~al.(2023)Zhang, Shen, Liu, Liu, Bendersky, Najork, and
  Zhang]{zhang2023not}
Rongzhi Zhang, Jiaming Shen, Tianqi Liu, Jialu Liu, Michael Bendersky, Marc
  Najork, and Chao Zhang.
\newblock Do not blindly imitate the teacher: Using perturbed loss for
  knowledge distillation.
\newblock \emph{arXiv preprint arXiv:2305.05010}, 2023.

\bibitem[Maas et~al.(2011)Maas, Daly, Pham, Huang, Ng, and
  Potts]{maas-EtAl:2011:ACL-HLT2011}
Andrew~L. Maas, Raymond~E. Daly, Peter~T. Pham, Dan Huang, Andrew~Y. Ng, and
  Christopher Potts.
\newblock Learning word vectors for sentiment analysis.
\newblock In \emph{Proceedings of the 49th Annual Meeting of the Association
  for Computational Linguistics: Human Language Technologies}, pages 142--150,
  Portland, Oregon, USA, June 2011.

\bibitem[Zhang et~al.(2015)Zhang, Zhao, and LeCun]{AGNews}
Xiang Zhang, Junbo Zhao, and Yann LeCun.
\newblock Character-level convolutional networks for text classification.
\newblock In \emph{NeurIPS}, pages 649--657, 2015.

\bibitem[{Alberto} et~al.(2015){Alberto}, {Lochter}, and {Almeida}]{youtube}
T.~C. {Alberto}, J.~V. {Lochter}, and T.~A. {Almeida}.
\newblock Tubespam: Comment spam filtering on youtube.
\newblock In \emph{ICMLA}, pages 138--143, 2015.

\bibitem[Li and Roth(2002)]{trec}
Xin Li and Dan Roth.
\newblock Learning question classifiers.
\newblock In \emph{COLING}, 2002.

\bibitem[Davis et~al.(2017)Davis, Grondin, Johnson, Sciaky, King, McMorran,
  Wiegers, Wiegers, and Mattingly]{davis2017comparative}
Allan~Peter Davis, Cynthia~J Grondin, Robin~J Johnson, Daniela Sciaky,
  Benjamin~L King, Roy McMorran, Jolene Wiegers, Thomas~C Wiegers, and
  Carolyn~J Mattingly.
\newblock The comparative toxicogenomics database: update 2017.
\newblock \emph{Nucleic acids research}, 45:\penalty0 D972--D978, 2017.

\bibitem[Hendrickx et~al.(2010)Hendrickx, Kim, Kozareva, Nakov, S{\'e}aghdha,
  Pad{\'o}, Pennacchiotti, Romano, and Szpakowicz]{hendrickx2010semeval}
Iris Hendrickx, Su~Nam Kim, Zornitsa Kozareva, Preslav Nakov, Diarmuid~{\'O}
  S{\'e}aghdha, Sebastian Pad{\'o}, Marco Pennacchiotti, Lorenza Romano, and
  Stan Szpakowicz.
\newblock Semeval-2010 task 8: Multi-way classification of semantic relations
  between pairs of nominals.
\newblock In \emph{Semeval}, pages 33--38, 2010.

\bibitem[Dawid and Skene(1979)]{dawid1979maximum}
Alexander~Philip Dawid and Allan~M Skene.
\newblock Maximum likelihood estimation of observer error-rates using the em
  algorithm.
\newblock \emph{Journal of the Royal Statistical Society: Series C (Applied
  Statistics)}, 28\penalty0 (1):\penalty0 20--28, 1979.

\bibitem[Li et~al.(2019)Li, Rubinstein, and Cohn]{li2019exploiting}
Yuan Li, Benjamin Rubinstein, and Trevor Cohn.
\newblock Exploiting worker correlation for label aggregation in crowdsourcing.
\newblock In \emph{International conference on machine learning}, pages
  3886--3895. PMLR, 2019.

\bibitem[Devlin et~al.(2018)Devlin, Chang, Lee, and Toutanova]{devlin2018bert}
Jacob Devlin, Ming-Wei Chang, Kenton Lee, and Kristina Toutanova.
\newblock Bert: Pre-training of deep bidirectional transformers for language
  understanding.
\newblock \emph{arXiv preprint arXiv:1810.04805}, 2018.

\bibitem[Loshchilov and Hutter(2017)]{loshchilov2018decoupled}
Ilya Loshchilov and Frank Hutter.
\newblock Decoupled weight decay regularization.
\newblock \emph{arXiv preprint arXiv:1711.05101}, 2017.

\end{thebibliography}
\clearpage
\appendix
\begin{table*}[!htb]
    \centering
    \caption{Detailed statistics of classification datasets.}
    \scalebox{0.75}{
    \begin{tabular}{ l l l c c c c c c c c c c}
        \toprule
         &
         \multicolumn{5}{c}{} &
         \multicolumn{4}{c}{\textbf{Avr. over LFs}}    &
         \multicolumn{1}{c}{\textbf{Train}}    &
         \multicolumn{1}{c}{\textbf{Valid}}    &
         \multicolumn{1}{c}{\textbf{Test}}   \\
         \cmidrule(lr){7-10}
         \cmidrule(lr){11-11}
         \cmidrule(lr){12-12}
         \cmidrule(lr){13-13}
           \textbf{Task} &\textbf{Domain} & \textbf{Dataset} & \textbf{\#Label} & \textbf{\#LF} & \textbf{Ovr. \%Coverage}& \textbf{\%Coverage} & \textbf{\%Overlap} & \textbf{\%Conflict} & \textbf{\%Accuracy} & \textbf{\#Data} & \textbf{\#Data} & \textbf{\#Data}   \\
         \midrule
     
    \multirow{2}{*}{Sentiment Classification}    

     & Movie & IMDb~\cite{maas-EtAl:2011:ACL-HLT2011}  & 2 & 5 & 87.58 & 23.60 & 11.60 & 4.50 & 69.88 & 20,000 & 2,500 & 2,500 \\      
     
     & Review & Yelp~\cite{AGNews}  & 2 & 8 & 82.78 & 18.34 & 13.58 & 4.94 & 73.05 & 30,400  & 3,800 & 3,800      \\\midrule
     
    \multirow{1}{*}{Spam Classification }    & Review & Youtube~\cite{youtube}  & 2 & 10 & 87.70 & 16.34 & 12.49 & 7.14 & 83.16 & 1,586 & 120 & 250 \\      
        
    \multirow{1}{*}{Topic Classification}    & News & AGNews~\cite{AGNews}  & 4 & 9 & 69.08 & 10.34 & 5.05 & 2.43 & 81.66 & 96,000 & 12,000 & 12,000  \\\midrule
    
    \multirow{1}{*}{Question Classification}    & Web Query & TREC~\cite{trec}  & 6 & 68 & 95.13 & 2.55 & 1.82 & 0.84 & 75.92 & 4,965 & 500 & 500  \\\midrule
    
     \multirow{2}{*}{Relation Classification}
     
     & Biomedical & CDR~\cite{davis2017comparative, ratner2017snorkel}  & 2 & 33 & 90.72 & 6.27 & 5.36 & 3.21 & 75.27 & 8,430 & 920 & 4,673  \\     
     & Web Text & SemEval~\cite{hendrickx2010semeval}  & 9 & 164 & 100.00 & 0.77 & 0.32 & 0.14 & 97.69 & 1,749 & 200 & 692   \\
    \bottomrule
    \end{tabular}
    }
    \label{tab:dataset_stats_class}
\end{table*}

\section{Data Statistics}
In Table~\ref{tab:dataset_stats_class}, we provide detailed information about datasets used in our experiments, encompassing statistics related to the labeling functions for each dataset.

\vspace{10pt}
\section{Implementation Details}
\label{app:impl}
The hyperparameters involved in this study includes the $k$ in Eq.~\eqref{eq:k} and $c_1$ in Eq.~\eqref{eq:c1}. The value for these two hyperparameters is listed in the Table~\ref{tab:hyperparameter}.

\begin{table}[!h]
\centering
\fontsize{8.2}{9}\selectfont 
\setlength{\tabcolsep}{0.3em}
\caption{Hyper-parameter Settings.}
\begin{tabular}{@{}lccccccc@{}}
\toprule
                     & IMDb & Yelp & Youtube &  AGNews & TREC & CDR & SemEval \\ \midrule
Max Length           & 150         & 256     & 256  & 512      & 64   & 128 &   128 \\
Learning Rate        & 1e-4        & 1e-4    & 1e-4 & 2e-5     & 5e-6 & 5e-6  & 3e-5    \\
Batch Size           & 64          & 32      & 32   & 16       & 64   & 32     & 32  \\
$k$    & 10 & 10 & 5 & 10 & 5 & 5 & 5 \\
$c_1$  &  4.0 & 4.0 & 8.0 & 4.0 & 10.0 & 10.0 & 10.0 \\ \bottomrule
\end{tabular}
\label{tab:hyperparameter}
\end{table}

\subsection{Computational Environment}
All of the experiments are conducted on \emph{CPU}: Intel(R) Core(TM) i7-5930K CPU @ 3.50GHz and \emph{GPU}: NVIDIA GeForce RTX A5000 GPUs with 24 GB memory using python 3.6 and Pytorch 1.10. 

\section{Proof}
Here, we provide evidence for Proposition \ref{prop:paradox}, which highlights that a good fit with weakly labeled data doesn't guarantee strong performance from their convex combination.
\begin{proof}
    Consider the weak sources $LF_1$ and $LF_2$, a dataset $X=\left\{x_1, x_2\right\}$, and the label space $Y = \{+1, -1\}$. The data samples get matched $LF_1$ and $LF_2$, respectively, and the weak labels given by the weak sources are
    \begin{equation}
        LF_1(x_1) = +1,\quad
        LF_2(x_2) = -1.
    \end{equation}
    From each weak source, we have base learner
    \begin{equation}
         f_1(x)=LF_1(x)=+1,\quad
        f_2(x)=LF_2(x)=-1.
    \end{equation}
    We have the convex combination of the two base learners 
    \begin{equation}\label{eq:convex}
    \begin{aligned}
        &\mathcal{L}\left(\frac{1}{2}\left(\mathcal{D}_1+\mathcal{D}_2\right), \alpha f_1+(1-\alpha) f_2\right)\\
        = & \frac{1}{2} \mathbb{I}(\alpha f_1(x_1)+(1-\alpha)f_2(x_1)) + 
        \frac{1}{2} \mathbb{I}(\alpha f_1(x_2)+(1-\alpha)f_2(x_2)) \\
        = & \frac{1}{2} \mathbb{I}(2 \alpha-1 \leq 0)+\frac{1}{2} \mathbb{I}(1-2 \alpha \leq 0) \geq \frac{1}{2}.
    \end{aligned}
    \end{equation}
    That is to say, even if the base learners perform well for the corresponding weak source such that
    \begin{equation}
    \begin{aligned}
        \mathcal{L}\left(\mathcal{D}_1, f_1\right)=\mathbb{I}\left(f_1\left(x_1\right) LF_1\left(x_1\right) \leq 0\right)=0,\\
        \mathcal{L}\left(\mathcal{D}_2, f_2\right)=\mathbb{I}\left(f_2\left(x_2\right) LF_2\left(x_2\right) \leq 0\right)=0,
    \end{aligned}
    \end{equation}
    the standard convex combination of them can still lead to a poor ensemble as $\mathcal{L}\left(\frac{1}{2}\left(\mathcal{D}_1+\mathcal{D}_2\right), \alpha f_1+(1-\alpha) f_2\right) \geq \frac{1}{2}$ shown in Eq.~\ref{eq:convex}.
\end{proof}
\label{app:proof}

\appendix

\end{document}